\documentclass[opre,nonblindrev]{informs3aa}
\OneAndAHalfSpacedXI

\usepackage{endnotes}
\let\footnote=\endnote
\let\enotesize=\normalsize
\def\notesname{Endnotes}
\def\makeenmark{$^{\theenmark}$}
\def\enoteformat{\rightskip0pt\leftskip0pt\parindent=1.75em
  \leavevmode\llap{\theenmark.\enskip}}
\usepackage{natbib}
\bibpunct[, ]{(}{)}{,}{a}{}{,}
\def\bibfont{\small}

\usepackage{soul}
\usepackage{url}
\usepackage{booktabs}
\usepackage{algorithm}
\usepackage{algorithmic}
\usepackage{mathtools}
\usepackage{caption}
\usepackage{subcaption}
\usepackage{bbm,bm}
\usepackage[table]{xcolor}
\usepackage{dsfont}
\usepackage{multirow}
\usepackage{makecell}
\usepackage{comment}

\usepackage{amsmath,amsfonts,bm}

\def\eqref#1{equation~\ref{#1}}
\def\1{\bm{1}}

\DeclareMathAlphabet{\mathsfit}{\encodingdefault}{\sfdefault}{m}{sl}
\SetMathAlphabet{\mathsfit}{bold}{\encodingdefault}{\sfdefault}{bx}{n}

\graphicspath{{\projectroot/}}

\newcommand{\mnorm}[2]{\left\lVert #1 \right\rVert_{#2}}
\newcommand{\algo}{RSM-LinUCB}
\newcommand{\model}{LBM}
\newcommand{\gmodel}{GLBM}
\newcommand{\galgo}{RSM-GLMUCB}
\TheoremsNumberedThrough
\EquationsNumberedThrough
\theoremstyle{TH}

\theoremstyle{EX}
\renewcommand{\qed}{\leavevmode\unskip\nobreak\hfill\mbox{\quad Q.E.D.}}
\let\informsproof\proof
\let\informsendproof\endproof
\renewenvironment{proof}[1][Proof]
  {\informsproof{#1.}}
  {\qed\informsendproof}

\usepackage[colorlinks=true,breaklinks=true,bookmarks=false,
  urlcolor=blue,citecolor=blue,linkcolor=blue]{hyperref}
\usepackage{cleveref}
\crefname{prop}{proposition}{propositions}
\Crefname{prop}{Proposition}{Propositions}
\crefname{cor}{corollary}{corollaries}
\Crefname{cor}{Corollary}{Corollaries}

\usepackage{color-edits}
\addauthor{hj}{blue}
\addauthor{jl}{orange}

\begin{document}
\allowdisplaybreaks

\RUNAUTHOR{Liu and Jia}
\RUNTITLE{Bellman-Centric Learning: Linear Bandits with Memory}
\TITLE{Bellman-Centric Learning: Near-Optimal Regret for Linear Bandits with Memory}
\ARTICLEAUTHORS{%
  \AUTHOR{Jingyuan Liu, Huiwen Jia}
  \AFF{Department of Industrial Engineering and Operations Research, University of California, Berkeley}
}

\ABSTRACT{%
We study linear bandits with memory, where past actions induce endogenous nonstationarity through an arbitrary known, bounded matrix-valued memory map. To trade off exploration and exploitation while accounting for the memory dynamics, we develop \algo{}, a Bellman-centric algorithm that learns as in linear bandits and plans as in reinforcement learning.
This design admits a novel regret decomposition which separates the memory-induced error from the cumulative reward estimation error along the learner’s trajectory. We prove a high-probability regret bound of $\widetilde O\big(dRS(M+1)+\sigma d\sqrt T\big)$, where $T$ is the learning horizon, $d$ is the parameter dimension, $M$ is the memory length, $R$ and $S$ bound the memory-map operator norm and reward-parameter norm, respectively, and $\sigma$ is the sub-Gaussian noise scale. Our results reveal that the multiplicative memory–horizon coupling in prior bounds is not intrinsic: memory only contributes an additive cost, up to logarithmic factors. We also prove a matching minimax lower bound, establishing near-optimality. We further extend the algorithm to generalized linear rewards, preserving this separation with near-optimal memory and leading statistical dependence. Our algorithms outperform the baselines in numerical experiments on synthetic instances and semi-synthetic KV- and semantic-cache tasks.
}
\maketitle

\section{Introduction}\label{sec:intro}

Linear bandits~\citep{AbbasiYadkori2011ImprovedAF, hao2020adaptive} are a widely studied framework for
sequential decision-making that generalizes multi-armed bandits~\citep{bubeck2012regret, auer2002finite} by
modeling expected rewards as linear functions of action features. Recent work ~\citep{wu2018learning,
chen2019new, luo2018efficient, di2020linear} extends linear bandits to the non-stationary setting, where the
unknown reward parameter may change over time. Such temporal variation arises in applications including large
language model (LLM) serving~\citep{taberner2026paretobandit, lee2026drift}, recommendation
systems~\citep{10494875, xu2020contextual}, and dynamic pricing~\citep{zhao2026high}. In these environments,
the learner must estimate rewards from noisy feedback while adapting to changes in the underlying environment.

Early work on non-stationary bandits focuses on exogenous non-stationarity, where the environment changes
independently of the learner's actions~\citep{garivier2011upper, besbes2014stochastic, karnin2016multi}. In
many applications, however, the learner's recent decisions themselves influence subsequent rewards, giving
rise to endogenous non-stationarity. One concrete example is key--value (KV) caching in LLM serving with a
fixed retention window \citep{qin2025mooncake}. Refreshing a prefix makes its KV tensors available for reuse
in the current and subsequent rounds. The benefit of a refresh therefore depends on earlier decisions:
refreshing an uncached prefix can enable reuse for additional requests, whereas refreshing an already cached
prefix extends its availability. Related memory effects arise in model or adapter routing
\citep{fu2024serverlessllm}, where recent pinning and routing decisions determine which models or adapters
remain warm. In recommendation and advertising \citep{warlop2018fighting,agrawal2020learning}, recent
exposures modify users' responses through priming, boredom, or wear-out effects; for example, repeatedly
displaying similar advertisements may cause user fatigue and reduce subsequent rewards. In dynamic pricing,
past prices affect subsequent revenues through finite-memory reference price effects \citep{den2022dynamic} or
through their impact on resource occupancy and availability \citep{jia2024online}.

Recent work on linear bandits with memory (\model{})~\citep{clerici2023linear, anngeneralized} studies one
specific endogenous non-stationarity setting: the most recent $M$ actions determine a known memory matrix
$A(\cdot)$ that transforms a fixed unknown reward parameter. Their algorithms divide the learning horizon into
episodes with fixed length and select an entire sequence of actions at the start of each episode by maximizing
the optimistic reward objective. The best known regret bound with this approach is
$\widetilde O\big(RS\sqrt{dMT}+\sigma\max\{1,S\}\sqrt{d\max\{d,R^2\}T}\big)$, where $T$ is the horizon, $M$ is
the memory length, $d$ is the reward-parameter dimension, $\sigma$ is the sub-Gaussian noise, and $R$ and $S$
bound the operator norm of $A(\cdot)$ and the Euclidean norm of $\theta^\star$, respectively. The first regret
term couples memory length and horizon through $\sqrt{MT}$, while the second introduces additional dependence
on $R$ and $S$ relative to the usual $\widetilde O(\sigma d\sqrt{T})$ statistical term. A similar
memory--horizon term appears in the extension to generalized linear bandits with memory
(\gmodel{}~\citep{anngeneralized}). At the meantime, the existing literature does not yet provide minimax
regret lower bounds that quantify the cost of memory. Thus, it remained an open question: is the
multiplicative memory--horizon dependence intrinsic, or can near-optimal regret be achieved with an additive
memory cost?

Memory-dependent rewards in these \model{} and \gmodel{} settings pose two main challenges. First, policy
updates must account for persistent memory. Although new observations allow the learner to improve the policy,
each updated policy inherits a memory state shaped by earlier decisions. Frequent updates may therefore
accumulate memory-induced error, requiring a tradeoff between incorporating new information and maintaining
policy stability. Second, learning must account for the long-term effects of actions. An action determines not
only the immediate reward but also the memory state governing future rewards. The learning agent must
therefore account for these long-term effects while learning the unknown reward parameter. Our key observation
is that the most recent $M$ actions in \model{} naturally form a state with known and deterministic
transitions, as in reinforcement learning (RL), while rewards are governed by a shared unknown parameter, as
in linear bandits. Neither view alone suffices: a myopic bandit approach neglects future memory effects and
can incur linear regret, whereas a direct tabular RL algorithm does not exploit the shared reward structure
and yields a regret with exponential dependence on the memory length. We therefore ask:
\begin{quote}
    \textit{Can we exploit this structure to learn from bandit feedback while planning over memory states as
    in RL, and achieve near-optimal regret?}
\end{quote}

We develop algorithms and regret guarantees for \model{} and \gmodel{}. Our contributions are as follows.
\begin{itemize}
    \item \textbf{We develop a Bellman-centric algorithm, \algo{} (Algorithm~\ref{al1}), in
    Section~\ref{sec:algo}.} It constructs optimistic reward estimates using linear-bandit confidence sets and
    uses Bellman planning over memory states to account for the long-term effects of actions. To control the
    memory-induced costs of policy updates, \algo{} adopts determinant-triggered rare
    switching~\citep{AbbasiYadkori2011ImprovedAF}: it updates optimistic reward estimates and policy only when
    the determinant of the information matrix doubles. We further develop \galgo{} for \gmodel{} in
    Section~\ref{sec:glbm}, extending optimistic reward estimation, Bellman planning, and rare policy updates
    to generalized linear rewards.

    \item \textbf{We establish near-optimal regret guarantees for \algo{} in Section~\ref{sec:theory}.} Our
    analysis admits a Bellman regret decomposition that separates memory-induced boundary costs from
    cumulative reward-estimation error along the learner's trajectory. Specifically,
    Theorem~\ref{th_upper_bound} gives an upper bound of $\widetilde O\big(dRS(M+1)+\sigma d\sqrt{T}\big)$.
    Compared with prior bounds, this replaces multiplicative memory--horizon dependence with an additive
    memory cost, while removing polynomial dependence on $R$ and $S$ from the leading statistical term.
    \textbf{The matching lower bound in Theorem~\ref{th_lower_bound} establishes optimality up to logarithmic
    factors.} For \gmodel{}, we further obtain
    $\widetilde O\big(d(M+1)\Delta_\mu+\sigma d\sqrt{T}+d^2\kappa v\rho^2\Delta_\mu\big)$ in
    Section~\ref{sec:glbm}. Here, $\Delta_\mu$ is the range of reward means, $\kappa$ is the worst-case
    inverse slope of the link function, and $v$ and $\rho$ govern reward dispersion and relative slope
    variation, respectively. This bound retains the separation between memory and statistical costs, with an
    additional curvature-induced error. The lower bound in Theorem~\ref{thm:glbm_lower} shows that the first
    two terms are near-optimal. Detailed comparisons of theoretical results are in
    Table~\ref{tab:comparisons}.

    \item \textbf{We evaluate our algorithms on synthetic instances and two semi-synthetic caching tasks in
    LLM serving.} Synthetic experiments examine performance over $T$ and sensitivity to $M$, $R$, and $S$,
    with \algo{} reducing cumulative regret by at least $61\%$ relative to the baselines. For KV caching, we
    use the Mooncake conversation trace~\citep{qin2025mooncake} to evaluate prefix-cache refresh decisions. On
    this task, \algo{} reduces cumulative regret by at least $20\%$. For semantic caching, we use
    vCache/SemBenchmarkLmArena~\citep{schroeder2025vcache}. The feedback is exponential with a reciprocal mean
    function, yielding a generalized linear reward model. On this task, \galgo{} reduces cumulative regret by
    at least $46\%$ relative to the baselines.

\end{itemize}

\begin{table}[htbp]
\centering
\caption{Comparison of regret bounds for the basic \model{} setting and the extended \gmodel{} setting.
}\label{tab:comparisons}
\resizebox{\textwidth}{!}{%
\begin{tabular}{@{}ccc@{}}
\toprule
\textbf{Comparisons for \model{}} & \textbf{Algorithm} & \textbf{Regret} \\
\midrule

\citet{JMLR:v11:jaksch10a} & \begin{tabular}[c]{@{}c@{}} UCRL2\end{tabular} & $\widetilde{O}\big( RSM
n^{M+1/2}\sqrt{T} + \sigma n^{(M+1)/2}\sqrt{T} \big)$ \\

\citet{clerici2023linear} & \begin{tabular}[c]{@{}c@{}} OFUL-memory\end{tabular} & $\widetilde{O}\big(
R(S+\sigma)\sqrt{dM}T^{3/4} \big)$ \\

\citet{anngeneralized} & OFUL-memory (improved) & $\widetilde{O}\big( RS\sqrt{dMT} +
\sigma\max\{1,S\}\sqrt{d\max\{d,R^2\}T} \big)$ \\

\rowcolor{gray!15}
\begin{tabular}[c]{@{}c@{}} \textbf{Ours} \\ \textbf{(Upper Bound)} \end{tabular} &
\begin{tabular}[c]{@{}c@{}} \algo{} \\ Theorem \ref{th_upper_bound} \end{tabular}
& $\widetilde{O}\big( dRSM + \sigma d\sqrt{T} \big)$ \\

\rowcolor{gray!15}
\begin{tabular}[c]{@{}c@{}} \textbf{Ours} \\ \textbf{(Lower Bound)} \end{tabular} & Theorem
\ref{th_lower_bound} & $\widetilde{\Omega}\big( dRSM + \sigma d\sqrt T \big)$ \\
\toprule
\textbf{Comparisons for \gmodel{}} & \textbf{Algorithm} & \textbf{Regret} \\
\midrule
\begin{tabular}[c]{@{}c@{}}\citet{anngeneralized}
\end{tabular}& GLBM-SCB &
$\widetilde O\big( \tau_{w} +\Delta_\mu\sqrt{MT} +C_\mu\sigma d\sqrt T +
\Gamma_\mu d^2\kappa v\rho^2\Delta_\mu \big)$ \\

\rowcolor{gray!10}
\begin{tabular}[c]{@{}c@{}}\textbf{Ours} \\ \textbf{(Upper Bound)} \end{tabular} & \begin{tabular}[c]{@{}c@{}}
\galgo{} \\ Theorem \ref{thm:glbm_upper} \end{tabular}
& $\widetilde{O}\big( dM\Delta_{\mu} + \sigma d\sqrt{T} + d^2\kappa v\rho^2\Delta_{\mu} \big)$ \\

\rowcolor{gray!10}
\begin{tabular}[c]{@{}c@{}} \textbf{Ours} \\ \textbf{(Lower Bound)} \end{tabular} & Theorem
\ref{thm:glbm_lower} &
$\widetilde{\Omega}\big( dM\Delta_{\mu} + \sigma d\sqrt{T} \big)$\\

\bottomrule\end{tabular}}
\begin{minipage}{\textwidth}
\small
\setlength{\parindent}{0pt}\raggedright

\end{minipage}
\par\smallskip
\begin{minipage}{\linewidth}
\footnotesize
\raggedright
\emph{Notes.} We assume $M \geq 1$ as the models reduce to standard (generalized) linear bandits when $M=0$.
Here, $n=|\mathcal{A}|$ denotes the action set size, and $\tau_w$ denotes the cumulative warm-up regret in
\citet{anngeneralized}, which may not be sublinear in $T$ when insufficient feature coverage prevents the
warm-up stopping condition from being met. For $\rho>0$, $C_\mu$ and $\Gamma_\mu$ are parameter-dependent
factors satisfying $C_\mu\geq(1+\rho RS)^{3/2}$ and $\Gamma_\mu\gtrsim d+(MT)^{1/4}$. For \gmodel{}, $\sigma$
denotes an upper bound on the conditional standard deviation.
\end{minipage}
\end{table}

\section{Problem Setup}\label{sec:setting}

We consider linear bandits with memory (\model{}). Let $[T]\coloneqq\{1,\ldots,T\}$, and let
$\mathcal{A}\subseteq\mathbb{R}^d$ be an action set satisfying $\sup_{a\in\mathcal{A}}\|a\|_2\leq 1$. At each
round $t\in[T]$, the learner selects an action $a_t\in\mathcal{A}$ based on the observed history and receives
a noisy reward $y_t$. The reward depends on the current action, a fixed parameter
$\theta^\star\in\mathbb{R}^d$, and the memory state $s_t$. The memory state
$s_t\coloneqq(a_{t-M},\ldots,a_{t-1})$ records the preceding $M$ actions, where \(M\) is the known memory
length. We denote the memory state space by $\mathcal{S}\coloneqq\mathcal{A}^M$ and fix an initial memory
state $s_1=(a_{1-M},\ldots,a_0)$. For notation simplicity, we use $f(s_t,a_t)=s_{t+1}=\left( a_{t-M+1},
\cdots, a_t \right)$ to denote the memory state transition. Specifically, the observed reward is
\begin{equation}\label{eq:reward_function}
    y_t = \left\langle a_t,A(s_t)\theta^\star\right\rangle +\varepsilon_t.
\end{equation}
The reward parameter $\theta^\star$ is unknown and satisfies
$\theta^\star\in\Theta\coloneqq \{\theta\in\mathbb{R}^d:\|\theta\|_2\leq S\}$. The memory map
$A:\mathcal{S}\to\mathbb{R}^{d\times d}$ is known a priori and uniformly bounded in operator norm: $
\sup_{s\in\mathcal{S}}\|A(s)\|_{\mathrm{op}}\leq R. $ The noise sequence $(\varepsilon_t)_{t\geq1}$ is
conditionally zero-mean and $\sigma$-sub-Gaussian given the history and the current action. Defining the
effective feature $\phi(s,a)\coloneqq A(s)^\top a$, we can express the conditional mean reward as
\[
    r(s,a) = \left\langle a,A(s)\theta^\star\right\rangle = \left\langle\phi(s,a),\theta^\star\right\rangle.
\]

The learner aims to maximize the cumulative mean reward. We compare the learner with the optimal length-$T$
action sequence, whose cumulative mean reward is
\begin{align}\label{align_opt}
    \operatorname{OPT}_T(\theta^\star) \coloneqq \sup_{\bar a_1,\ldots,\bar a_T\in\mathcal{A}} \sum_{t=1}^T
    \left\langle \bar a_t, A(\bar a_{t-M},\ldots,\bar a_{t-1})\theta^\star \right\rangle.
\end{align}
Finally, for a policy $\pi$ generating actions $a_1^\pi,\ldots,a_T^\pi$, we define its policy regret as
\begin{equation}
    \operatorname{Reg}_T(\pi) \coloneqq \operatorname{OPT}_T(\theta^\star) - \sum_{t=1}^T
    \left\langle a_t^\pi, A(a_{t-M}^\pi,\ldots,a_{t-1}^\pi)\theta^\star \right\rangle.
\end{equation}

\begin{remark}[General memory matrix.]
Prior work on \model{}~\citep{clerici2023linear,anngeneralized} primarily considers the structured additive
memory map \(A(s_t)=\big(I+\sum_{k=1}^{M}a_{t-k}a_{t-k}^{\top}\big)^\gamma\) and points out that their methods
can be be generalized to a general map. In this paper, we relax this structural restriction by assuming a
general map \(A:\mathcal{S}\to\mathbb{R}^{d\times d}\) with only the boundedness condition
\(\sup_{s\in\mathcal{A}^M}\|A(s)\|_{\mathrm{op}}\le R,\) thereby accommodating a broader and more general
class of order-sensitive and cross-action memory dynamics.
\end{remark}

\section{Bellman-Centric Algorithm Design}\label{sec:algo}

We introduce \textit{Rare-Switch Memory LinUCB} (\algo{}; see Algorithm~\ref{al1}). This algorithm is
motivated by the observation that \model{} lies at the interface between stochastic linear bandits and
reinforcement learning (RL). The most recent $M$ actions in \model{} naturally form a state with known,
deterministic transitions, while rewards are governed by a shared unknown parameter, as in linear bandits.
Thus, \model{} requires non-myopic planning, as in RL, but only the reward parameter needs to be learned from
bandit feedback.

From the bandit side, we construct an optimistic reward function \(u_k(s,a)\) using a linear bandit confidence
set. From the RL side, we solve the corresponding Bellman equation to obtain a value function \(J_k(s)\) and
select actions by accounting for both immediate optimism and the value of the resulting memory state. This
non-myopic planning step goes beyond the batching-based treatment of memory used in prior \model{}
algorithms~\citep{clerici2023linear,anngeneralized}. To prevent frequent policy changes from repeatedly
inducing memory-state transients, we adopt a determinant-triggered rare-switching
rule~\citep{AbbasiYadkori2011ImprovedAF} and replan only after the accumulated statistical information has
increased substantially.

\begin{algorithm}[htbp]
\caption{Rare-Switch Memory LinUCB}\label{al1}
\begin{algorithmic}[1]
    \STATE \textbf{Input:} Confidence level $\delta\in(0,1)$, regularization $\lambda>0$, discount factor
    $\gamma$

    \STATE \textbf{Initialize:}
    \(V_0\gets \lambda I_d,\
    b_0\gets 0_d,\ t\leftarrow 1,\ k\gets 0\)
    \WHILE{$t\le T$}
        \STATE $k\gets k+1,\ \tau_k\gets t,\ \overline V_k\gets V_{t-1},\ \overline b_k\gets b_{t-1},\ \overline\theta_k
        \gets V_{t-1}^{-1}b_{t-1}$
        \WHILE{$t\le T$ \textbf{ and }$\det(V_{t-1})<2\det(\overline V_k)$}
            \STATE Play $a_t$ according to Equation (\ref{align_maximize_action_value_function}) with
            $u_k(s,a)$ and $J_{k}(s)$ defined as in Equations (\ref{align_calculate_reward_estimate}) and
            (\ref{align_calculate_value_function}), and observe reward $y_t$
            \STATE $V_t\gets V_{t-1}+\phi(s_t,a_t)\phi(s_t,a_t)^\top,\ b_t\gets b_{t-1}+y_t\phi(s_t,a_t),\ s_{t+1}\gets f(s_t,a_t),\ t\gets t+1$
        \ENDWHILE
    \ENDWHILE
\end{algorithmic}
\end{algorithm}

\paragraph{Input and initialization.} The algorithm takes a confidence level \(\delta\), a regularization
parameter \(\lambda\), and a discount factor \(\gamma\) as inputs. It initializes the information matrix and
reward feature vector as \(V_0=\lambda I_d\) and \(b_0=0_d\), respectively, and sets the episode counter
\(k=0\). The algorithm proceeds in episodes, during each of which the reward estimator and the resulting
policy remain fixed to limit the frequency of policy changes and to control the cost of the memory state
transients.

\paragraph{Action selection with frozen parameters.} At the beginning of episode \(k\), the algorithm records
the starting round \(\tau_k\) and freezes the current information matrix \(\overline V_k\), reward feature
vector \(\overline b_k\), and parameter estimate \(\overline\theta_k\). At round \(t\) within episode \(k\),
the algorithm selects an action that maximizes the optimistic action value, given by
\begin{align}\label{align_maximize_action_value_function}
    a_t \in \arg\max_{a\in\mathcal A} \left\{ u_k(s_t,a) + \gamma J_k(f(s_t,a)) \right\},
\end{align}
where \(u_k(s,a)\) denotes the optimistic reward estimate and \(J_k(s)\) the corresponding optimal discounted
value, defined as
\begin{align}
    & u_k(s,a) = \Pi_{[-RS,RS]} \big( \phi(s,a)^\top\overline\theta_k + \overline\beta_k
    \mnorm{\phi(s,a)}{\overline V_k^{-1}}\big),\label{align_calculate_reward_estimate}
    \\& J_k(s) = \sup_{\overline a_0,\overline a_1,\ldots} \sum_{j=0}^{\infty} \gamma^j
    u_k(\overline s_j,\overline a_j).\label{align_calculate_value_function}
\end{align}
Here,
\(\overline\beta_k=\sqrt{\lambda}S+\sigma\sqrt{2\log\big(\frac{1}{\delta}\big)+d\log\big(1+\frac{\tau_kR^2}{\lambda d}\big)}\)
is the confidence level, \(\Pi_{[-RS,RS]}\) projects the estimate to the range of rewards, and
\(\overline s_0=s\), \(\overline s_{j+1}=f(\overline s_j,\overline a_j)\) is the memory-state trajectory. With
high probability, \(u_k(s,a)\) upper bounds the expected reward of taking action \(a\) in memory state \(s\).
Built on this optimistic reward function, \(J_k(s)\) accounts for the optimistic cumulative reward obtainable
from state \(s\) by recursively evaluating future action sequences and their induced memory-state transitions.
The algorithm plays \(a_t\) and observes the reward \(y_t\).

\paragraph{Determinant-triggered policy updates.} The agent updates \(V_t\), \(b_t\), and the memory state
\(s_{t+1}\) using the selected action and the observed reward. The algorithm uses a determinant-triggered
rare-switching rule, starting a new episode only when \(\det(V_t)\geq2\det(\overline V_k)\). Thus, the
estimator used for planning and the resulting policy remain fixed within each episode, even as new
observations continue to accumulate. The selected actions may nevertheless vary as the memory state evolves.
This update rule limits the number of policy changes and the memory state transients. Therefore, any two
initial memories at the start of episode become identical after the same sequence of \(M\) actions, allowing
us to bound the transient contribution through an \(M\)-dependent value span. Unlike the predetermined
batching policies used in prior work on \model{}~\citep{clerici2023linear,anngeneralized}, our episode
boundaries adapt to the information accumulated by the learner. The determinant-doubling rule therefore
controls the frequency of replanning while ensuring that each policy update is triggered by a substantial
increase in statistical information.

\begin{remark}[Computation of $J_k$]
Action selection in Equation (\ref{align_maximize_action_value_function}) requires the value function $J_k$,
which is fixed per episode; within the episode, each round reduces to a maximization over a single action. For
small-scale problems where the memory state space $\mathcal S$ is not large, $J_k$ can be computed
near-exactly by discounted Howard policy iteration, and we show in Appendix~\ref{app_howard_policy_iteration}
that the resulting approximation error is negligible, so the regret order is preserved as in
Section~\ref{sec:theory}. For large-scale problems, the planning step may use numerical optimization methods
such as gradient ascent, which have also been used for optimistic arm selection in \model{}~\citep[see,
e.g.,][]{clerici2023linear,anngeneralized}. Moreover, application structure often admits more efficient
planners: for example, in our cache experiments in Section~\ref{sec:exp}, the memory state is determined by
the set of refreshed items, so planning reduces to a combinatorial optimization that can be solved exactly by
tree dynamic programming or sorting (see Appendix~\ref{app_exp} for the details).
\end{remark}

\section{Theoretical Guarantee}\label{sec:theory}

\subsection{Upper Bound of \algo{}}
Our analysis first establishes that, under any bounded reward function, the optimal discounted values of two
memory states differ by at most $2RSM$.

\begin{lemma}[Finite-memory value span]\label{lemma:value_span}
Let $\gamma\in(0,1)$ and $r_q:\mathcal S\times\mathcal A\to[-RS,RS]$ be any reward function. Define
$J_q(s)=\sup_{a_0,a_1,\ldots} \sum_{j=0}^{\infty}\gamma^j r_q(s_j,a_j)$, where $s_0=s$ and
$s_{j+1}=f(s_j,a_j)$. Then
\[
    \operatorname{span}(J_q) \coloneqq \sup_{s\in\mathcal S}J_q(s) - \inf_{s\in\mathcal S}J_q(s) \leq 2RSM.
\]
\end{lemma}

The key insight behind Lemma~\ref{lemma:value_span} is that applying the same sequence of $M$ actions from any
two initial memory states makes the resulting states coincide. The initial history at each episode can
therefore affect rewards during the first $M$ rounds, which bounds the effect of the starting state at each
episode boundary. Combining this property with the Bellman optimality equation and optimism yields the
following regret decomposition.

\begin{lemma}[Bellman-centric regret decomposition]
\label{lemma:regret_decomposition}
Consider Algorithm~\ref{al1} with $T\geq2$ and $\gamma=1-1/T$. Let $\pi$ denote the policy generated by
\Cref{al1}, $K$ the total number of episodes, and $k(t)$ the episode containing round $t$. Then, with
probability at least $1-\delta$,
\begin{align*}
    \operatorname{Reg}_T(\pi) \leq 2(K+2)RSM+ \sum_{t=1}^{T} \left[ u_{k(t)}(s_t^\pi,a_t^\pi) -
    r_{\theta^\star}(s_t^\pi,a_t^\pi) \right].
\end{align*}
\end{lemma}

Lemma~\ref{lemma:regret_decomposition} separates the regret into a \textit{memory-induced error} and the
cumulative reward \textit{estimation error} along the learner's trajectory. The first term bounds the error of
starting each episode from a potentially unfavorable memory state. Our determinant-doubling rule ensures
$K=O(d\log T)$, so this error contributes only an additive memory-dependent term to the regret bound. The
second term can be controlled by using ridge-regression confidence bounds and the elliptical potential
lemma~\citep{AbbasiYadkori2011ImprovedAF}. Specifically, the determinant doubling trick ensures that freezing
the estimator within each episode increases the confidence intervals by at most a constant factor.

\begin{theorem}[Upper Bound of Algorithm~\ref{al1}]\label{th_upper_bound}
For $T\geq2$, with inputs $0<\lambda\leq d\sigma^2/S^2$, and $\gamma=1-1/T$, the regret of Algorithm~\ref{al1}
satisfies
\begin{small}
\begin{align*}
    \textnormal{Reg}_T(\textnormal{RSM-LinUCB}) &\leq \left(4d\log\left( 1 + \frac{TR^2}{\lambda d} \right) +
    6 \right)RSM + 4RSd\log\left( 1+\frac{TR^2}{\lambda d} \right)
    \\& + 4\sqrt{Td\log\left( 1+\frac{TR^2}{\lambda d} \right)}\left( \sigma\sqrt{d\log\left( 1 +
    \frac{TR^2}{\lambda d} \right) + 2\log\left( \frac{1}{\delta} \right)} + \sqrt\lambda S \right)
    \\&=\widetilde O \big( dRS(M+1)+\sigma d\sqrt{T} \big)
\end{align*}
\end{small}
with probability at least $1-\delta$, where $\widetilde O$ omits logarithmic factors.
\end{theorem}

The regret bound consists of two additive terms. The first, $\widetilde O(dRS(M+1))$, captures the
memory-induced error of starting each episode from a potentially unfavorable state. The finite memory
value-span bound limits this error to $O(RSM)$ per episode, while determinant doubling limits the number of
episodes to $\widetilde O(d)$. The additional $\widetilde O(dRS)$ term accounts for the initial exploration
error present even without memory that also exists in ordinary linear bandits. The second term,
$\widetilde O(\sigma d\sqrt T)$, accounts for reward estimation error. Under the prescribed regularization,
this term has no polynomial dependence on the parameter-scale factors $R$ (for memory matrix $A$) or $S$ (for
reward parameter $\theta^*$).

\begin{remark}
Compared with the best upper bound
$\widetilde O\big(RS\sqrt{dMT}+\sigma\max\{1,S\}\sqrt{d\max\{d,R^2\}T}\big)$ from prior
work~\citep{anngeneralized}, our result not only clearly separates the memory-dependent cost from the
$\sqrt T$ estimation error term, but also removes the additional parameter-scale factors ($R$ and $S$) from
the latter (see Table~\ref{tab:comparisons} for a detailed comparison). Thus, larger $M$ or $R$ only increases
the additive memory term rather than the coefficient of the leading $\sqrt T$ term, up to logarithmic factors.
This separation is particularly beneficial when the memory length $M$ is large. The lower bound below
establishes that these dependencies are near-optimal up to logarithmic factors.
\end{remark}

\subsection{Lower Bound}

\begin{theorem}[Lower Bound of LBM]
\label{th_lower_bound}
Suppose $d\geq2$ and \(T\geq d\max\left\{ M,1,\sigma^2d/\left(48R^2S^2\right) \right\}.\) There exists an
admissible action set such that the minimax regret of \model{} satisfies
\[
    \inf_{\pi} \sup_{\substack{ \theta^\star:\ \|\theta^\star\|_2\leq S\\
    A:\ \|A\|_{\mathrm{op}}\leq R }} \mathbb E\left[\operatorname{Reg}_T(\pi)\right]
    \geq \Omega\big(dRS(M+1)+\sigma d\sqrt T\big).
\]
\end{theorem}

The upper bound in Theorem \ref{th_upper_bound} matches the lower bound in Theorem \ref{th_lower_bound} up to
logarithmic factors. The first memory-induced error reflects the cost of obtaining information that is
accessible only through a suitable action history. In our hard instance, the learner must repeatedly select a
coordinate direction for $M$ rounds before rewards can reveal whether that coordinate is informative.
Identifying the rewarding coordinate among $d$ candidates can therefore require $\Omega(dM)$ rounds, with a
reward gap of order $RS$ during unsuccessful exploration. This yields the $\Omega(dRSM)$ term, while the
additional $\Omega(dRS)$ term accounts for the initial exploration error already present in ordinary linear
bandits, giving $\Omega(dRS(M+1))$ overall. The statistical estimation error $\Omega(\sigma d\sqrt T)$ follows
from the standard lower bound in linear bandits~\citep{lattimore2020bandit}, since ordinary linear bandits are
a special case of \model{} with a constant memory matrix.

\section{Extension to Generalized Linear Bandits with Memory}\label{sec:glbm}

Our Bellman-centric design achieves near-optimal regret for \model{}. We next ask whether the same idea can
improve prior bounds for generalized linear rewards~\citep{anngeneralized} and yield near-optimal leading
regret terms, as summarized in Table~\ref{tab:comparisons}. The extension is not immediate: in generalized
linear models the slope of the link function depends on the unknown parameter, which makes uncertainty
estimation considerably harder than in the linear case. In this section, we extend our algorithm to
generalized rewards, show that it achieves an additive memory error, and prove that the leading terms of the
resulting bound are near-optimal.

\subsection{Problem Setup}

Generalized linear bandits with memory (\gmodel{}) retain the same action space $\mathcal A$, known memory
matrix $A(\cdot)$, and norm bounds $R$ and $S$ as in \model{}. Similarly, we write $\phi(s,a)=A(s)^\top a$ and
$z(s,a)=\phi(s,a)^\top\theta^*$. Unlike \model{}, where the conditional reward mean is $z(s,a)$, \gmodel{}
models the reward through a known canonical exponential family~\citep{lee2024glm, liu2024almost,
anngeneralized}. Specifically, an observation $y$ at $(s,a)$ has density or probability mass function
\[
    \mathbb{P}[y\mid z(s,a)] = \exp\left( \frac{yz(s,a)-b(z(s,a))}{v}+c(y,v) \right),
\]
where $v>0$ is the known dispersion parameter which controls the noise level of the observed reward, and
$\mu(\cdot)=b'(\cdot)$ is the known link function. The conditional reward mean and variance are therefore
given by $\mathbb{E}[y\mid z(s,a)]=\mu(z(s,a))$ and $\operatorname{Var}[y\mid z(s,a)]=v\mu'(z(s,a))$,
respectively.

Following prior work, we assume that $\mu$ is twice continuously differentiable and satisfies $\mu'(z)>0$ and
$|\mu''(z)|\le\rho\mu'(z)$ for every $|z|\le RS$, where the known constant $\rho\ge0$ controls the relative
variation of its slope. We further define
\begin{align}\label{align_definition_in_glbm}
    U_\mu \coloneqq \sup_{|z|\le RS}\mu'(z),
    \
    \kappa \coloneqq \Big(\inf_{|z|\le RS}\mu'(z)\Big)^{-1},
    \
    \Delta_\mu \coloneqq \mu(RS)-\mu(-RS).
\end{align}
Here, $U_\mu$ bounds how fast the reward changes as $z(s,a)$ varies, $\kappa$ quantifies the worst-case
flatness of the link function, and $\Delta_\mu$ bounds the range of reward means. The definitions of
$\mathrm{OPT}_T$ and $\mathrm{Reg}_T$ remain unchanged from Section~\ref{sec:setting}, with the mean reward
replaced by $r(s,a)=\mu(\phi(s,a)^\top\theta^\star)$.

\subsection{Algorithm and Theoretical Guarantee}

In \gmodel{}, the slope of the mean reward function depends on the unknown parameter, making uncertainty
estimation more challenging than in the linear model. We therefore propose \textit{Rare-Switch Memory GLMUCB}
(\galgo{}; see Algorithm~\ref{al2}), which uses maximum likelihood estimation (MLE) to construct confidence
sets $\mathcal C_k$ at each episode while retaining the discounted planning approach of Algorithm~\ref{al1}.
The algorithm intersects each new confidence set with the previous one so that uncertainty does not increase
across episodes. We also maintain two information matrices $V_t$ and $H_t$: the former one uses all selected
actions, while the latter one weights them by the estimated reward slope when predictions are sufficiently
accurate. These weights, similar to the rare-switch policy in Algorithm~\ref{al1}, are fixed once recorded,
and we update the policy only when either determinant doubles. This allows the algorithm to account for
generalized rewards while keeping policy updates infrequent. Due to space limits, we defer further details of
Algorithm~\ref{al2} to Appendix~\ref{app_glbm}.

\begin{algorithm}[htbp]
\caption{Rare-Switch Memory GLMUCB}
\label{al2}
\begin{algorithmic}[1]
    \STATE \textbf{Input:} Confidence level $\delta\in(0,1)$; regularization $\lambda_1,\lambda_2>0$; discount
    factor $\gamma\in(0,1)$.
    \STATE \textbf{Initialize:} $V_0\gets\lambda_1 I_d$, $H_0\gets\lambda_2 I_d$, $\mathcal C_0\gets\Theta$,
    $t\gets1$, $k\gets0$.
    \WHILE{$t\le T$}
        \STATE $k\gets k+1$, $\tau_k\gets t$, $\overline V_k\gets V_{t-1}$, $\overline H_k\gets H_{t-1}$.
        \STATE Compute $\widehat\theta_k$, $\mathcal C_k$, and $\overline\theta_k$ according to
        \Cref{eq:glbm_confidence_set}.
        \WHILE{$t\le T$ and $\max\left\{\frac{\det(V_{t-1})}{\det(\overline V_k)},
        \frac{\det(H_{t-1})}{\det(\overline H_k)}\right\}< 2$}
            \STATE Play $a_t$ according to Equation~(\ref{eq:glbm_policy}) with $u_k(s,a)$ and $J_k(s)$
            defined as in Equation~(\ref{eq:glbm_optimistic_reward}), and observe $y_t$. Then compute $w_t$
            according to Equation~(\ref{eq:glbm_curvature_weight}).
            \STATE $V_t=V_{t-1}+\phi(s_t,a_t)\phi(s_t,a_t)^\top,\ H_t=H_{t-1}+w_t\phi(s_t,a_t)\phi(s_t,a_t)^\top$,
            $s_{t+1}\gets f(s_t,a_t)$, $t\gets t+1$.
        \ENDWHILE
    \ENDWHILE
\end{algorithmic}
\end{algorithm}

\begin{theorem}[Upper bound of Algorithm~\ref{al2}]
\label{thm:glbm_upper}
For $T\ge2$, \( 0<\lambda_1\leq\kappa vd/S^2, 0<\lambda_2\leq vd/S^2, \gamma=1-1/T. \) Under generalized
reward, with probability at least $1-\delta$, the regret of Algorithm~\ref{al2} satisfies
\[
    \operatorname{Reg}_T(\textnormal{\galgo{}}) \le \widetilde O\big( d(M+1)\Delta_\mu + d\sqrt{vU_\mu T} +
    d^2\kappa v\rho^2\Delta_\mu \big).
\]
\end{theorem}

The first is the memory-induced error: over $\widetilde O(d)$ episodes, the $M$-dependent transient cost is at
most $M\Delta_\mu$ per episode, with an initial exploration cost of \(\widetilde O(d\Delta_\mu)\). The second
term accounts for parameter estimation error, where \(\sqrt{vU_\mu}\) bounds the conditional standard
deviation of the observation noise, denoted by $\sigma$ in Table~\ref{tab:comparisons}. These two terms have
the similar interpretation as their counterparts in Theorem~\ref{th_upper_bound}. The third term captures the
additional cost of learning under parameter-dependent curvature. Here, \(\kappa\) reflects how small the
derivative of the link function can be, while \(\rho\) controls its relative variation as the predicted reward
parameter changes. Our analysis bounds the number of rounds with large prediction uncertainty or insufficient
information by \(\widetilde O(d^2\kappa v\rho^2+d)\), and charges at most \(\Delta_\mu\) estimation error to
each such round. When \(\rho=0\), the link function is affine and the curvature dependence vanishes, and this
bound reduces to Theorem~\ref{th_upper_bound} under linear reward.

\begin{theorem}[Lower bound of \gmodel{}]
\label{thm:glbm_lower}
Let $d\geq3$, $M\geq 0$, and
\(T\geq\max\big\{dM,d,\frac{v(d-1)^2}{U_\mu}\max\left\{\frac{1}{4R^2S^2},\rho^2\right\}\big\}\). Then there
exists an admissible action set such that
\begin{equation}\label{eq:glbm_lower_bound}
    \inf_{\pi}\sup_{\theta^\star,A} \mathbb E\left[\mathrm{Reg}_T(\pi)\right] \geq \widetilde\Omega\Big(
    d(M+1)\Delta_\mu +d\sqrt{vU_\mu T} \Big).
\end{equation}
\end{theorem}
This lower bound matches the first two terms of the upper bound in Theorem~\ref{thm:glbm_upper} up to
logarithmic factors. Thus, both the memory-induced error and the estimation error are unavoidable in
\gmodel{}, consistent with our findings for \model{}. The remaining gap is the curvature-induced
identification cost: our upper bound controls it through worst-case curvature, while the lower bound does not
establish whether this dependence is necessary. Closing this gap remains an interesting direction. Compared
with ~\citep{anngeneralized}, our upper bound offers three improvements. It avoids a separate warm-up phase
whose regret can be linear in $T$; isolates the dependence on $M$ in an additive memory term while retaining
the optimal $\sqrt{T}$ estimation rate; and removes the additional $M$- and $T$-dependence from the
curvature-induced error while improving its dimension dependence (see Table~\ref{tab:comparisons}).

\section{Numerical Experiments}\label{sec:exp}

We evaluate our algorithms on synthetic data and semi-synthetic datasets for KV- and semantic caching. For the
LBM experiments, including the synthetic and KV-cache experiments, we compare \algo{} (ours) with OFUL
(original)~\citep{clerici2023linear}, its refined-block variant (OFUL (refined)), GLBM-SCB adapted to linear
rewards~\citep{anngeneralized}, and Myopic LinUCB~\citep{AbbasiYadkori2011ImprovedAF}. For the GLBM experiment
on semantic caching, we compare \galgo{} (ours) with GLBM-SCB and the per-round GLB-OMD
algorithm~\citep{zhang2026generalized}. We measure cumulative pseudo-regret against the exact finite horizon
optimum under the same initial memory. More implementation details and parameter choices are provided in
Appendix~\ref{app_exp}.

\paragraph{Synthetic data experiments.} We consider \model{} instances with $d=10$,
$\mathcal{A}=\{e_1,\ldots,e_{10}\}$, and independent Gaussian noise of standard deviation $\sigma=0.25$. The
memory matrix is $R e_i e_i^\top$ if the preceding $M$ actions are all $e_i$, and zero otherwise, so switching
actions delays informative rewards. We set $\theta^\star=Su$ for a fixed unit direction $u$ and use
$\gamma=0.9999$. We record a continuous trajectory up to $T=10^6$ and separately vary
$M\in\{100,200,\ldots,1000\}$ and $R,S\in\{10,20,\ldots,100\}$, one parameter at a time. Results average 30
noise seeds, with ten balanced initial histories averaged within each seed, and are presented in
Figure~\ref{fig:placeholder} (A) - (D). After initial exploration, \algo{} achieves the lowest regret among
all algorithms across the tested values of $M$, $R$, and $S$.

\paragraph{KV cache.} We use Mooncake conversation dataset \citep{qin2025mooncake}, containing 12,031 LLM
requests whose anonymized hashes identify shared 512-token prefix blocks. Each action refreshes one candidate
prefix, whose KV blocks remain available for the current and following $M$ rounds. The reward is the
normalized average number of reused prefix tokens across 36 held-out requests, a proxy for avoided prefill
computation; its conditional expectation is linear in the post-refresh cache-availability feature, yielding an
LBM instance. We simulate refresh decisions under a $3\times10^6$-token cache budget, where each round
represents 10 seconds and $M=180$ corresponds approximately to Mooncake Store's default 30-minute soft-pin
lifetime~\citep{mooncakeStore}. To reduce communication overhead and limit exposure of individual access
patterns~\citep{lmcacheMetrics,wu2025know}, we consider a lightweight refresh controller that receives
aggregate cache-reuse statistics rather than request-level prefix logs. We set $T=10^5$, $M=180$, and choose
$M\in\{30,60,90,120,150,180\}$, using paired request streams and three seeds, see Figure~\ref{fig:placeholder}
(E) - (F). \algo{} achieves lower regret in the time comparison and the lowest mean regret throughout the
memory sweep. We also note that at large $M$, regret increases slower, which is consistent with the fact that
longer retention keeps more useful prefixes resident and reduces losses from missed reuse.

\paragraph{Semantic cache.} We use \texttt{vCache/SemBenchmarkLmArena} ~\citep{schroeder2025vcache}, which
contains real LMArena queries, generated paraphrases, responses, embeddings, and semantic labels. Unlike exact
prefix reuse in KV caching, semantic caching reuses complete answers and may return an unsuitable response.
Effective semantic caching must therefore reduce model calls while maintaining reuse
correctness~\citep{bang-2023-gptcache,schroeder2025vcache}. We hence measure how long the cache sustains
correct, cache-only service before the first model fallback or incorrect reuse. Each round, the controller
refreshes one bundle without observing incoming queries, which then remains available for the current and
following $M$ rounds. We simulate unit-rate Poisson arrivals with independently sampled queries and use
benchmark labels to assess correctness. For each simulated waiting-time measurement, we hold the resulting
cache configuration fixed. If $z=\phi(s,a)^\top\theta^\star\in[0,1)$ denotes the rate of requests correctly
served by the cache, this duration is exponentially distributed with mean $\mu(z)=1/(1-z)$, yielding a
generalized linear reward model. With a fixed 64 MiB payload budget, we report a $T=10^6$ run at $M=51$ and a
sweep over $M\in\{10,20,24,30,40,50\}$, using three seeds, see Figure~\ref{fig:placeholder} (G) - (H).
\galgo{} outperforms both baselines and achieves lower mean regret at every tested value of $M$. At larger
$M$, its regret grows more slowly, consistent with longer retention improving cache coverage.

\begin{figure}[htbp]
    \centering
    \includegraphics[width=\linewidth]{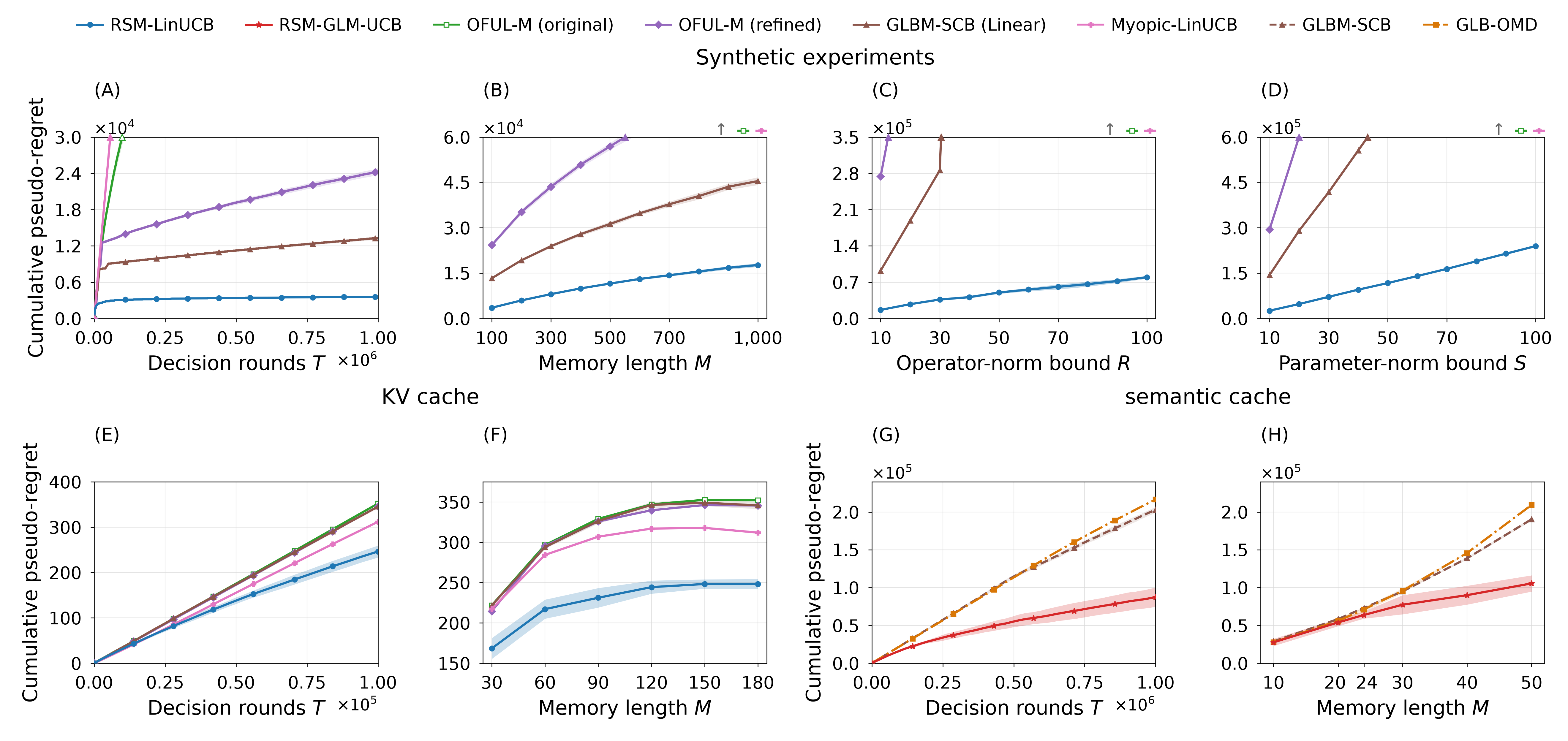}
    \caption{Cumulative pseudo-regret for synthetic (A--D), KV-cache (E--F), and semantic-cache (G--H).
    Shading denotes one standard deviation and some baseline curves exceed the displayed range.}
    \label{fig:placeholder}
\end{figure}

\section{Conclusion}
We developed \algo{} and \galgo{} for \model{} and \gmodel{}, combining linear bandit reward estimation and
determinant-triggered rare switching with Bellman planning over memory states. Our analysis separates
memory-induced error from estimation error, with leading regret terms matching the lower bounds up to
logarithmic factors. Future work includes closing the gap in the curvature-induced term for \gmodel{} by
sharpening the upper bound or proving a matching lower bound. Another direction is to explore whether our
action-dependent memory framework can help develop theoretical models of memory management in LLM-based
systems.

%\clearpage
% Print endnotes if future edits to the shared sections introduce footnotes.
\ifnum\value{endnote}>0
  \theendnotes
\fi

\bibliographystyle{informs2014}
\bibliography{lr}

@inproceedings{AbbasiYadkori2011ImprovedAF,
  title={Improved Algorithms for Linear Stochastic Bandits},
  author={Yasin Abbasi-Yadkori and D{\'a}vid P{\'a}l and Csaba Szepesvari},
  booktitle={Neural Information Processing Systems},
  year={2011}
}

@book{lattimore2020bandit,
  title={Bandit algorithms},
  author={Lattimore, Tor and Szepesv{\'a}ri, Csaba},
  year={2020},
  publisher={Cambridge University Press}
}

@article{clerici2023linear,
  title={Linear bandits with memory: from rotting to rising},
  author={Clerici, Giulia and Laforgue, Pierre and Cesa-Bianchi, Nicolo},
  journal={arXiv preprint arXiv:2302.08345},
  year={2023}
}

@inproceedings{anngeneralized,
  title={Generalized Linear Bandits with Memory},
  author={Ann, Heesang and Choi, Hyunjun and Hwang, Taehyun and Shin, Younghoon and Cheong, Haeju and Oh, Min-hwan},
  booktitle={Forty-third International Conference on Machine Learning},
  year={2026}
}

@inproceedings{garivier2011upper,
  title={On upper-confidence bound policies for switching bandit problems},
  author={Garivier, Aur{\'e}lien and Moulines, Eric},
  booktitle={International conference on algorithmic learning theory},
  pages={174--188},
  year={2011},
  organization={Springer}
}

@article{besbes2014stochastic,
  title={Stochastic multi-armed-bandit problem with non-stationary rewards},
  author={Besbes, Omar and Gur, Yonatan and Zeevi, Assaf},
  journal={Advances in neural information processing systems},
  volume={27},
  year={2014}
}

@InProceedings{pmlr-v99-auer19a,
  title = 	 {Adaptively Tracking the Best Bandit Arm with an Unknown Number of Distribution Changes},
  author =       {Auer, Peter and Gajane, Pratik and Ortner, Ronald},
  booktitle = 	 {Proceedings of the Thirty-Second Conference on Learning Theory},
  pages = 	 {138--158},
  year = 	 {2019},
  volume = 	 {99},
  series = 	 {Proceedings of Machine Learning Research},
  month = 	 {25--28 Jun},
  publisher =    {PMLR}
}

@inproceedings{wu2018learning,
  title={Learning contextual bandits in a non-stationary environment},
  author={Wu, Qingyun and Iyer, Naveen and Wang, Hongning},
  booktitle={The 41st International ACM SIGIR Conference on Research \& Development in Information Retrieval},
  pages={495--504},
  year={2018}
}

@inproceedings{chen2019new,
  title={A new algorithm for non-stationary contextual bandits: Efficient, optimal and parameter-free},
  author={Chen, Yifang and Lee, Chung-Wei and Luo, Haipeng and Wei, Chen-Yu},
  booktitle={Conference on Learning Theory},
  pages={696--726},
  year={2019},
  organization={PMLR}
}

@article{di2020linear,
  title={A linear bandit for seasonal environments},
  author={Di Benedetto, Giuseppe and Bellini, Vito and Zappella, Giovanni},
  journal={arXiv preprint arXiv:2004.13576},
  year={2020}
}

@inproceedings{xu2020contextual,
  title={Contextual-bandit based personalized recommendation with time-varying user interests},
  author={Xu, Xiao and Dong, Fang and Li, Yanghua and He, Shaojian and Li, Xin},
  booktitle={Proceedings of the AAAI Conference on Artificial Intelligence},
  volume={34},
  number={04},
  pages={6518--6525},
  year={2020}
}

@inproceedings{li2021unifying,
  title={Unifying clustered and non-stationary bandits},
  author={Li, Chuanhao and Wu, Qingyun and Wang, Hongning},
  booktitle={International Conference on Artificial Intelligence and Statistics},
  pages={1063--1071},
  year={2021},
  organization={PMLR}
}

@article{karnin2016multi,
  title={Multi-armed bandits: Competing with optimal sequences},
  author={Karnin, Zohar S and Anava, Oren},
  journal={Advances in Neural Information Processing Systems},
  volume={29},
  year={2016}
}

@inproceedings{luo2018efficient,
  title={Efficient contextual bandits in non-stationary worlds},
  author={Luo, Haipeng and Wei, Chen-Yu and Agarwal, Alekh and Langford, John},
  booktitle={Conference On Learning Theory},
  pages={1739--1776},
  year={2018},
  organization={PMLR}
}

@inproceedings{cheung2019learning,
  title={Learning to optimize under non-stationarity},
  author={Cheung, Wang Chi and Simchi-Levi, David and Zhu, Ruihao},
  booktitle={The 22nd International Conference on Artificial Intelligence and Statistics},
  pages={1079--1087},
  year={2019},
  organization={PMLR}
}

@article{russac2019weighted,
  title={Weighted linear bandits for non-stationary environments},
  author={Russac, Yoan and Vernade, Claire and Capp{\'e}, Olivier},
  journal={Advances in Neural Information Processing Systems},
  volume={32},
  year={2019}
}

@article{russac2020algorithms,
  title={Algorithms for non-stationary generalized linear bandits},
  author={Russac, Yoan and Capp{\'e}, Olivier and Garivier, Aur{\'e}lien},
  journal={arXiv preprint arXiv:2003.10113},
  year={2020}
}

@inproceedings{kim2020randomized,
  title={Randomized exploration for non-stationary stochastic linear bandits},
  author={Kim, Baekjin and Tewari, Ambuj},
  booktitle={Conference on Uncertainty in Artificial Intelligence},
  pages={71--80},
  year={2020},
  organization={PMLR}
}

@inproceedings{zhao2020simple,
  title={A simple approach for non-stationary linear bandits},
  author={Zhao, Peng and Zhang, Lijun and Jiang, Yuan and Zhou, Zhi-Hua},
  booktitle={International Conference on Artificial Intelligence and Statistics},
  pages={746--755},
  year={2020},
  organization={PMLR}
}

@InProceedings{pmlr-v206-hong23b,
  title = 	 {An Optimization-based Algorithm for Non-stationary Kernel Bandits without Prior Knowledge},
  author =       {Hong, Kihyuk and Li, Yuhang and Tewari, Ambuj},
  booktitle = 	 {Proceedings of The 26th International Conference on Artificial Intelligence and Statistics},
  pages = 	 {3048--3085},
  year = 	 {2023},
  volume = 	 {206},
  series = 	 {Proceedings of Machine Learning Research},
  month = 	 {25--27 Apr},
  publisher =    {PMLR}
}

@InProceedings{pmlr-v206-wang23k,
  title = 	 {Revisiting Weighted Strategy for Non-stationary Parametric Bandits},
  author =       {Wang, Jing and Zhao, Peng and Zhou, Zhi-Hua},
  booktitle = 	 {Proceedings of The 26th International Conference on Artificial Intelligence and Statistics},
  pages = 	 {7913--7942},
  year = 	 {2023},
  volume = 	 {206},
  series = 	 {Proceedings of Machine Learning Research},
  month = 	 {25--27 Apr},
  publisher =    {PMLR}
}

@article{JMLR:v11:jaksch10a,
  author  = {Thomas Jaksch and Ronald Ortner and Peter Auer},
  title   = {Near-optimal Regret Bounds for Reinforcement Learning},
  journal = {Journal of Machine Learning Research},
  year    = {2010},
  volume  = {11},
  number  = {51},
  pages   = {1563--1600},
}

@article{warlop2018fighting,
  title={Fighting boredom in recommender systems with linear reinforcement learning},
  author={Warlop, Romain and Lazaric, Alessandro and Mary, J{\'e}r{\'e}mie},
  journal={Advances in Neural Information Processing Systems},
  volume={31},
  year={2018}
}

@inproceedings{awasthi2022congested,
  title={Congested bandits: Optimal routing via short-term resets},
  author={Awasthi, Pranjal and Bhatia, Kush and Gollapudi, Sreenivas and Kollias, Kostas},
  booktitle={International Conference on Machine Learning},
  pages={1078--1100},
  year={2022},
  organization={PMLR}
}

@inproceedings{malik2022complete,
  title={Complete policy regret bounds for tallying bandits},
  author={Malik, Dhruv and Li, Yuanzhi and Singh, Aarti},
  booktitle={Conference on Learning Theory},
  pages={5146--5174},
  year={2022},
  organization={PMLR}
}

@inproceedings{malik2023weighted,
  title={Weighted tallying bandits: overcoming intractability via repeated exposure optimality},
  author={Malik, Dhruv and Igoe, Conor and Li, Yuanzhi and Singh, Aarti},
  booktitle={International Conference on Machine Learning},
  pages={23590--23609},
  year={2023},
  organization={PMLR}
}

@article{tang2021bandit,
  title={Bandit learning with delayed impact of actions},
  author={Tang, Wei and Ho, Chien-Ju and Liu, Yang},
  journal={Advances in Neural Information Processing Systems},
  volume={34},
  pages={26804--26817},
  year={2021}
}

@inproceedings{agrawal2020learning,
  title={Learning by repetition: Stochastic multi-armed bandits under priming effect},
  author={Agrawal, Priyank and Tulabandula, Theja},
  booktitle={Conference on Uncertainty in Artificial Intelligence},
  pages={470--479},
  year={2020},
  organization={PMLR}
}

@article{leqi2021rebounding,
  title={Rebounding bandits for modeling satiation effects},
  author={Leqi, Liu and Kilinc Karzan, Fatma and Lipton, Zachary and Montgomery, Alan},
  journal={Advances in Neural Information Processing Systems},
  volume={34},
  pages={4003--4014},
  year={2021}
}

@inproceedings{laforgue2022last,
  title={A last switch dependent analysis of satiation and seasonality in bandits},
  author={Laforgue, Pierre and Clerici, Giulia and Cesa-Bianchi, Nicol{\`o} and Gilad-Bachrach, Ran},
  booktitle={International Conference on Artificial Intelligence and Statistics},
  pages={971--990},
  year={2022},
  organization={PMLR}
}

@inproceedings{foussoul2023last,
  title={Last switch dependent bandits with monotone payoff functions},
  author={Foussoul, Ayoub and Goyal, Vineet and Papadigenopoulos, Orestis and Zeevi, Assaf},
  booktitle={International Conference on Machine Learning},
  pages={10265--10284},
  year={2023},
  organization={PMLR}
}

@article{khosravi2023preferences,
  title={Preferences Evolve And So Should Your Bandits: Bandits with Evolving States for Online Platforms},
  author={Khosravi, Khashayar and Leme, Renato Paes and Podimata, Chara and Tsorvantzis, Apostolis},
  journal={arXiv preprint arXiv:2307.11655},
  year={2023}
}

@inproceedings {qin2025mooncake,
author = {Ruoyu Qin and Zheming Li and Weiran He and Jialei Cui and Feng Ren and Mingxing Zhang and Yongwei Wu and Weimin Zheng and Xinran Xu},
title = {Mooncake: Trading More Storage for Less Computation {\textemdash} A {KVCache-centric} Architecture for Serving {LLM} Chatbot},
booktitle = {23rd USENIX Conference on File and Storage Technologies (FAST 25)},
year = {2025},
isbn = {978-1-939133-45-8},
address = {Santa Clara, CA},
pages = {155--170},
url = {https://www.usenix.org/conference/fast25/presentation/qin},
publisher = {USENIX Association},
month = feb
}

@misc{mooncakeStore,
  author       = {{Mooncake Team}},
  title        = {{Mooncake Store}: Design and Configuration Documentation},
  howpublished = {\url{https://kvcache-ai.github.io/Mooncake/design/mooncake-store.html}},
  note         = {Accessed: August 14, 2026},
  year         = {2026}
}

@article{zhang2026generalized,
  title={Generalized linear bandits: Almost optimal regret with one-pass update},
  author={Zhang, Yu-Jie and Xu, Sheng-An and Zhao, Peng and Sugiyama, Masashi},
  journal={Advances in Neural Information Processing Systems},
  volume={38},
  pages={69244--69277},
  year={2026}
}

@inproceedings{lee2024glm,
  title = {A Unified Confidence Sequence for Generalized Linear Models,
           with Applications to Bandits},
  author = {Lee, Junghyun and Yun, Se-Young and Jun, Kwang-Sung},
  booktitle = {Advances in Neural Information Processing Systems},
  volume = {37},
  year = {2024},
  publisher = {Curran Associates, Inc.},
}

@article{zhang2025linear,
  title={Linear bandits on ellipsoids: Minimax optimal algorithms},
  author={Zhang, Raymond and Hadiji, H{\'e}di and Combes, Richard},
  journal={arXiv preprint arXiv:2502.17175},
  year={2025}
}

@article{scherrer2013improved,
  title={Improved and generalized upper bounds on the complexity of policy iteration},
  author={Scherrer, Bruno},
  journal={Advances in Neural Information Processing Systems},
  volume={26},
  year={2013}
}

@article{schroeder2025vcache,
  title={vcache: Verified semantic prompt caching},
  author={Schroeder, Luis Gaspar and Desai, Aditya and Cuadron, Alejandro and Chu, Kyle and Liu, Shu and Zhao, Mark and Krusche, Stephan and Kemper, Alfons and Zaharia, Matei and Gonzalez, Joseph E},
  journal={arXiv preprint arXiv:2502.03771},
  year={2025}
}

@inproceedings{dani2008stochastic,
  title={Stochastic linear optimization under bandit feedback},
  author={Dani, Varsha and Hayes, Thomas P and Kakade, Sham M},
  booktitle={21st Annual Conference on Learning Theory},
  number={101},
  pages={355--366},
  year={2008}
}

@misc{vllmMetrics,
  author       = {{vLLM Contributors}},
  title        = {{vLLM} Documentation: Metrics},
  howpublished = {\url{https://docs.vllm.ai/en/v0.12.0/design/metrics/}},
  note         = {Version 0.12.0. Accessed September 15, 2026}
}

@InProceedings{pmlr-v54-abeille17a,
  title = 	 {{Linear Thompson Sampling Revisited}},
  author = 	 {Abeille, Marc and Lazaric, Alessandro},
  booktitle = 	 {Proceedings of the 20th International Conference on Artificial Intelligence and Statistics},
  pages = 	 {176--184},
  year = 	 {2017},
  volume = 	 {54},
  series = 	 {Proceedings of Machine Learning Research},
  month = 	 {20--22 Apr},
  publisher =    {PMLR}
}

@article{filippi2010parametric,
  title={Parametric bandits: The generalized linear case},
  author={Filippi, Sarah and Cappe, Olivier and Garivier, Aur{\'e}lien and Szepesv{\'a}ri, Csaba},
  journal={Advances in neural information processing systems},
  volume={23},
  year={2010}
}

@InProceedings{pmlr-v70-li17c,
  title = 	 {Provably Optimal Algorithms for Generalized Linear Contextual Bandits},
  author =       {Lihong Li and Yu Lu and Dengyong Zhou},
  booktitle = 	 {Proceedings of the 34th International Conference on Machine Learning},
  pages = 	 {2071--2080},
  year = 	 {2017},
  volume = 	 {70},
  series = 	 {Proceedings of Machine Learning Research},
  month = 	 {06--11 Aug},
  publisher =    {PMLR}
}

@InProceedings{pmlr-v119-faury20a,
  title = 	 {Improved Optimistic Algorithms for Logistic Bandits},
  author =       {Faury, Louis and Abeille, Marc and Calauzenes, Clement and Fercoq, Olivier},
  booktitle = 	 {Proceedings of the 37th International Conference on Machine Learning},
  pages = 	 {3052--3060},
  year = 	 {2020},
  volume = 	 {119},
  series = 	 {Proceedings of Machine Learning Research},
  month = 	 {13--18 Jul},
  publisher =    {PMLR}
}

@InProceedings{pmlr-v130-abeille21a,
  title = 	 { Instance-Wise Minimax-Optimal Algorithms for Logistic Bandits },
  author =       {Abeille, Marc and Faury, Louis and Calauzenes, Clement},
  booktitle = 	 {Proceedings of The 24th International Conference on Artificial Intelligence and Statistics},
  pages = 	 {3691--3699},
  year = 	 {2021},
  volume = 	 {130},
  series = 	 {Proceedings of Machine Learning Research},
  month = 	 {13--15 Apr},
  publisher =    {PMLR}
}

@article{liu2024almost,
  title={Almost Free: Self-concordance in Natural Exponential Families and an Application to Bandits},
  author={Liu, Shuai and Ayoub, Alex and Sentenac, Flore and Tan, Xiaoqi and Szepesv{\'a}ri, Csaba},
  journal={Advances in Neural Information Processing Systems},
  volume={37},
  pages={70828--70874},
  year={2024}
}

@article{sawarni2024generalized,
  title={Generalized linear bandits with limited adaptivity},
  author={Sawarni, Ayush and Das, Nirjhar and Barman, Siddharth and Sinha, Gaurav},
  journal={Advances in Neural Information Processing Systems},
  volume={37},
  pages={8329--8369},
  year={2024}
}

@article{jun2017scalable,
  title={Scalable generalized linear bandits: Online computation and hashing},
  author={Jun, Kwang-Sung and Bhargava, Aniruddha and Nowak, Robert and Willett, Rebecca},
  journal={Advances in Neural Information Processing Systems},
  volume={30},
  year={2017}
}

@InProceedings{pmlr-v151-faury22a,
  title = 	 { Jointly Efficient and Optimal Algorithms for Logistic Bandits },
  author =       {Faury, Louis and Abeille, Marc and Jun, Kwang-Sung and Calauzenes, Clement},
  booktitle = 	 {Proceedings of The 25th International Conference on Artificial Intelligence and Statistics},
  pages = 	 {546--580},
  year = 	 {2022},
  volume = 	 {151},
  series = 	 {Proceedings of Machine Learning Research},
  month = 	 {28--30 Mar},
  publisher =    {PMLR}
}

@inproceedings{hao2020adaptive,
  title={Adaptive exploration in linear contextual bandit},
  author={Hao, Botao and Lattimore, Tor and Szepesvari, Csaba},
  booktitle={International Conference on Artificial Intelligence and Statistics},
  pages={3536--3545},
  year={2020},
  organization={PMLR}
}

@article{bubeck2012regret,
  title={Regret analysis of stochastic and nonstochastic multi-armed bandit problems},
  author={Bubeck, S{\'e}bastien and Cesa-Bianchi, Nicolo},
  journal={Foundations and Trends{\textregistered} in Machine Learning},
  volume={5},
  number={1},
  pages={1--122},
  year={2012},
  publisher={Emerald Publishing Limited}
}

@article{auer2002finite,
  title={Finite-time analysis of the multiarmed bandit problem},
  author={Auer, Peter and Cesa-Bianchi, Nicolo and Fischer, Paul},
  journal={Machine learning},
  volume={47},
  number={2},
  pages={235--256},
  year={2002},
  publisher={Springer}
}

@article{taberner2026paretobandit,
  title={ParetoBandit: Budget-Paced Adaptive Routing for Non-Stationary LLM Serving},
  author={Taberner-Miller, Annette},
  journal={arXiv preprint arXiv:2604.00136},
  year={2026}
}

@article{lee2026drift,
  title={Drift-Aware LLM Routing with Sparse Contexts and Shared Budgets},
  author={Lee, Cheung Hao and Wong, Patrick},
  journal={arXiv preprint arXiv:2609.00662},
  year={2026}
}

@ARTICLE{10494875,
  author={Ghoorchian, Saeed and Kortukov, Evgenii and Maghsudi, Setareh},
  journal={IEEE Open Journal of Signal Processing}, 
  title={Non-Stationary Linear Bandits With Dimensionality Reduction for Large-Scale Recommender Systems}, 
  year={2024},
  volume={5},
  number={},
  pages={548-558},
  doi={10.1109/OJSP.2024.3386490}}

@article{zhao2026high,
  title={High-dimensional dynamic pricing under nonstationarity: Learning and earning with change-point detection},
  author={Zhao, Zifeng and Jiang, Feiyu and Yu, Yi and Chen, Xi},
  journal={Management Science},
  year={2026},
  publisher={INFORMS}
}

@inproceedings{fu2024serverlessllm,
  title={{ServerlessLLM}: {Low-Latency} serverless inference for large language models},
  author={Fu, Yao and Xue, Leyang and Huang, Yeqi and Brabete, Andrei-Octavian and Ustiugov, Dmitrii and Patel, Yuvraj and Mai, Luo},
  booktitle={18th USENIX Symposium on Operating Systems Design and Implementation (OSDI 24)},
  pages={135--153},
  year={2024}
}

@article{den2022dynamic,
  title={Dynamic pricing with demand learning and reference effects},
  author={den Boer, Arnoud V and Keskin, N Bora},
  journal={Management Science},
  volume={68},
  number={10},
  pages={7112--7130},
  year={2022},
  publisher={INFORMS}
}

@misc{lmcacheMetrics,
  author       = {{LMCache}},
  title        = {{LMCache} Documentation: Metrics Reference},
  howpublished = {\url{https://docs.lmcache.ai/production/observability/metrics.html}},
  note         = {Accessed September 18, 2026}
}

@inproceedings{wu2025know,
  title={I Know What You Asked: Prompt Leakage via KV-Cache Sharing in Multi-Tenant LLM Serving.},
  author={Wu, Guanlong and Zhang, Zheng and Zhang, Yao and Wang, Weili and Niu, Jianyu and Wu, Ye and Zhang, Yinqian},
  booktitle={NDSS},
  year={2025}
}

@inproceedings{bang-2023-gptcache,
    title = "{GPTC}ache: An Open-Source Semantic Cache for {LLM} Applications Enabling Faster Answers and Cost Savings",
    author = "Bang, Fu",
    booktitle = "Proceedings of the 3rd Workshop for Natural Language Processing Open Source Software (NLP-OSS 2023)",
    month = dec,
    year = "2023",
    address = "Singapore",
    publisher = "Association for Computational Linguistics",
    pages = "212--218"
}

@article{jia2024online,
  title={Online learning and pricing for service systems with reusable resources},
  author={Jia, Huiwen and Shi, Cong and Shen, Siqian},
  journal={Operations Research},
  volume={72},
  number={3},
  pages={1203--1241},
  year={2024},
  publisher={Informs}
}

@article{jia2022online_a,
  title={Online learning and pricing for network revenue management with reusable resources},
  author={Jia, Huiwen and Shi, Cong and Shen, Siqian},
  journal={Advances in Neural Information Processing Systems},
  volume={35},
  pages={4830--4842},
  year={2022}
}

@inproceedings{jia2022online_b,
  title={Online learning and pricing with reusable resources: Linear bandits with sub-exponential rewards},
  author={Jia, Huiwen and Shi, Cong and Shen, Siqian},
  booktitle={International Conference on Machine Learning},
  pages={10135--10160},
  year={2022},
  organization={PMLR}
}

\begin{APPENDICES}
\renewcommand{\theHsection}{appendix.\Alph{section}}
\section{Related works}

\textbf{Linear bandits with memory.} The works most closely related to ours are the linear bandits with memory
models studied by~\citet{clerici2023linear} and~\citet{anngeneralized}. \citet{clerici2023linear} introduce a
linear bandit model in which the learner's most recent actions endogenously affect the reward parameter. Their
OFUL-memory algorithm reduces the original problem to learning a cyclic policy: at the start of each episode,
the learner commits to a sequence of actions, with the first $M$ actions clearing the memory inherited from
the previous episode, and updates its parameter estimate only after executing the entire sequence. This
approach achieves a $\widetilde O(T^{3/4})$ regret bound against the optimal action sequence. Building on this
algorithmic idea, \citet{anngeneralized} develop a sharper episode-level analysis of OFUL-memory, yielding a
regret bound of $\widetilde O(RS\sqrt{dMT}+\sigma\max\{1,S\}\sqrt{d\max\{d,R^2\}T})$. However, this approach
still selects an entire action sequence at the start of each predetermined episode with fixed length, leaving
the memory length and horizon coupled in the regret bound. In contrast, our key observation is that the most
recent $M$ actions naturally form a state with known and deterministic transitions, while rewards are governed
by a shared unknown parameter $\theta^\star$, as in linear bandits. Building on this structure, \algo{}
combines optimistic reward estimation with Bellman planning and determinant-triggered rare switching. For
$M\geq1$ and a general known memory map satisfying $\|A(\cdot)\|_{\mathrm{op}}\leq R$, our algorithm achieves
$\widetilde O(dRSM+\sigma d\sqrt{T})$ regret. This separates the memory-induced error from the statistical
estimation error, making the memory cost additive and independent of the horizon up to logarithmic factors.
The bound matches our lower bound up to logarithmic factors. Compared with the refined OFUL-memory bound in
prior work, our $dRSM$ term improves on its $RS\sqrt{dMT}$ memory term, while our $\sigma d\sqrt{T}$
estimation term removes the additional polynomial dependence on $R$ and $S$. A detailed comparison is provided
in Table~\ref{tab:comparisons}.

\textbf{Non-stationary bandits.} Our work is related to the broader literature on non-stationary bandits.
Classical non-stationary bandit models characterize environmental changes through a bounded number of
breakpoints~\citep{garivier2011upper,pmlr-v99-auer19a}, with algorithms based on discounting, sliding windows,
restarts, or adaptive change detection. Other formulations control non-stationarity through a variation budget
or related regularity conditions~\citep{besbes2014stochastic,karnin2016multi}. These ideas have subsequently
been extended to linear and contextual
bandits~\citep{wu2018learning,chen2019new,di2020linear,xu2020contextual,li2021unifying,luo2018efficient,cheung2019learning,russac2019weighted,kim2020randomized,pmlr-v206-hong23b},
as well as to generalized linear reward models~\citep{russac2020algorithms,zhao2020simple,pmlr-v206-wang23k}.
In these models, the time-varying parameter $\theta_t$ is typically generated by an exogenous process that is
not controlled by the learner. By contrast, the non-stationarity in our setting is endogenous: the learner's
recent actions change the memory state and thereby affect future reward functions.

In addition to \model{}, several models allow past actions to influence current rewards. Recency-dependent
recommendation models~\citep{warlop2018fighting} express each arm's reward as a polynomial of recency-weighted
play counts over a finite action set. Congested bandits~\citep{awasthi2022congested} model each arm's reward
as depending on how often it was selected within a fixed window. Bandit-based pricing
models~\citep{jia2022online_a,jia2022online_b} study queueing systems with reusable resources, where price
changes induce transient effects before the system reaches a new steady state. Furthermore, related models
include tallying and weighted-tallying bandits~\citep{malik2022complete,malik2023weighted},delayed-impact
bandits~\citep{tang2021bandit}, priming and wear-in/wear-out effects~\citep{agrawal2020learning}, rebounding
or satiation dynamics~\citep{leqi2021rebounding}, last-switch-dependent
rewards~\citep{laforgue2022last,foussoul2023last}, and bandits with deterministically evolving
preferences~\citep{khosravi2023preferences}. These works typically impose specific forms on how past actions
affect rewards, using arm-wise history statistics or application-specific dynamics, such as recency-weighted
play counts, congestion effects, or queueing transients. In contrast, we consider a general matrix-valued
memory map $A(a_{t-M},\ldots,a_{t-1})$ that accommodates order-sensitive and cross-action interactions without
prescribing a particular memory mechanism, with regret bounds independent of the number of arms.

\paragraph{Generalized linear bandits.} Generalized linear bandits model expected rewards through a known link
function applied to a linear predictor. Early work developed optimistic algorithms based on maximum likelihood
estimation (MLE)~\citep{filippi2010parametric,pmlr-v70-li17c}. For logistic bandits,
\citet{pmlr-v119-faury20a} removed the worst case curvature dependence from the leading $\sqrt{T}$ term, while
\citet{pmlr-v130-abeille21a} refined the role of curvature through instance-dependent upper and lower bounds.
\citet{lee2024glm} developed likelihood-ratio confidence sequences for generalized linear models and applied
them to bandits. \citet{liu2024almost} studied self-concordance in natural exponential families and derived
regret guarantees for generalized linear bandits with subexponential reward tails. Computational efficiency
has also received considerable attention~\citep{jun2017scalable,pmlr-v151-faury22a}; in particular,
\citet{zhang2026generalized} used online mirror descent in place of MLE to obtain efficient one-pass updates.
Closely related to our update mechanism, \citet{sawarni2024generalized} studied generalized linear bandits
with limited adaptivity, including rarely-switching algorithms. These works address nonlinear reward learning,
but their regret guarantees do not account for the effects of current actions on future memory states. In
contrast, the \gmodel{} setting allows rewards to depend on the preceding $M$ actions and compares the learner
with the optimal complete action sequence. Our extension combines likelihood-based confidence sets with
Bellman planning and rare switching. Here, rare switching not only reduces replanning frequency but also
controls the memory-induced error at episode boundaries, allowing the finite memory structure to separate this
cost from statistical estimation error.

\section{Detailed Algorithmic Steps in GLBM}\label{app_glbm}

In this appendix, we provide the omitted details of Algorithm~\ref{al2} in \gmodel{}. We begin with the action
selection with frozen parameters in the algorithm.

\paragraph{Action selection with frozen parameters.} Like Algorithm~\ref{al1}, \galgo{} freezes its parameter
estimate at the beginning of each episode. For generalized rewards, however, we replace ridge regression with
maximum likelihood estimation. Let $\Theta=\{\theta\in\mathbb R^d:\|\theta\|_2\le S\}$. At the beginning of
episode $k$, we compute the maximum likelihood estimate $\widehat\theta_k$, construct a nested confidence set
$\mathcal C_k$ of possible values of $\theta^\star$, and select the point $\overline\theta_k$ in
$\mathcal C_k$ closest to $\widehat\theta_k$:
\begin{equation}\label{eq:glbm_confidence_set}
    \begin{aligned}
    \widehat\theta_k &\in\arg\min_{\theta\in\Theta}\mathcal L_{\tau_k-1}(\theta),\\
    \mathcal C_k
    &=\mathcal C_{k-1}\cap\left\{\theta\in\Theta:\mathcal L_{\tau_k-1}(\theta)-\mathcal L_{\tau_k-1}(\widehat\theta_k)\le vc_T(\delta)
    \right\},
    \\\overline\theta_k&= \arg\min_{\theta\in\mathcal C_k}\|\theta-\widehat\theta_k\|_2^2,
    \end{aligned}
\end{equation}
where $\mathcal L_t(\theta)$ measures how well $\theta$ explains the rewards observed in the first $t$ rounds.
It is obtained by multiplying the negative log-likelihood by $v$ and omitting terms independent of $\theta$,
so minimizing it yields the maximum likelihood estimate. The confidence threshold $c_T(\delta)$ determines the
allowed increase $v c_T(\delta)$ above this minimum when constructing the confidence set:
\begin{align}
    \mathcal L_t(\theta) &=\sum_{j=1}^{t}
    \left[b(\phi(s_j,a_j)^\top\theta)-y_j\phi(s_j,a_j)^\top\theta\right],
    \label{eq:glbm_likelihood}\\
    c_T(\delta) &=\log\left(\frac{2}{\delta}\right)+d+
    d\log\left(1+\frac{4RST}{vd}\left[U_\mu RS+\frac{\sqrt{vU_\mu}}{\delta}\right]\right).
    \label{eq:glbm_confidence_threshold}
\end{align}

If the intersection in Equation~(\ref{eq:glbm_confidence_set}) is empty, the algorithm sets
$\mathcal C_k=\mathcal C_{k-1}$, so the confidence sets remain nested and nonempty across episodes. After
constructing $\mathcal C_k$ and $\overline\theta_k$, at each round $t$ of episode $k$ the algorithm selects an
action that maximizes the optimistic action value, as in Algorithm~\ref{al1}:
\begin{equation}\label{eq:glbm_policy}
    a_t\in\arg\max_{a\in\mathcal A}\left\{u_k(s_t,a)+\gamma J_k(f(s_t,a))\right\},
\end{equation}
where $u_k(s,a)$ is the optimistic reward estimate and $J_k(s)$ is the corresponding optimal discounted value,
defined as
\begin{align}\label{eq:glbm_optimistic_reward}
    u_k(s,a) &= \max_{\theta\in\mathcal C_k}\mu\left(\phi(s,a)^\top\theta\right),
    \quad J_k(s)=\sup_{\overline a_0,\overline a_1,\ldots}\sum_{j=0}^{\infty}\gamma^j
    u_k(\overline s_j,\overline a_j).
\end{align}
Here, $\overline s_0=s$ and $\overline s_{j+1}=f(\overline s_j,\overline a_j)$. As in the linear model, this
planning step accounts for both the immediate optimistic reward and the effect of current actions on future
memory states. The algorithm plays $a_t$ and observes the reward $y_t$. As discussed in
Section~\ref{sec:algo}, $J_k$ can be approximated using discounted Howard policy iteration for small-scale
problems or numerical planning methods for larger problems. In our semantic cache experiment
(Section~\ref{sec:exp} and Appendix~\ref{app_exp}), we further exploit its disjoint-bundle structure,
combining a finite-vertex confidence approximation with sorting-based cache optimization and cyclic refreshes
to implement the optimistic planning step more efficiently (Appendix~\ref{app+semantic_cache}).

\paragraph{Determinant-triggered policy updates.} Unlike Algorithm~\ref{al1}, \galgo{} also uses the local
slope of the mean reward function when updating its information matrices. For the selected feature
$\phi(s_t,a_t)$, it computes
\begin{equation}\label{eq:glbm_curvature_weight}
    w_t=\mu'(\phi(s_t,a_t)^\top\overline\theta_k)\mathbf 1\left\{\rho\omega_k(\phi(s_t,a_t))\le 1 \right\},
\end{equation}
where
$\omega_k(\phi(s,a))=\max_{\theta\in\mathcal C_k}\phi(s,a)^\top\theta-\min_{\theta\in\mathcal C_k}\phi(s,a)^\top\theta$
denotes the prediction width for the feature $\phi(s,a)$. Under the property that $|\mu''(z)|\leq\rho\mu'(z)$,
the width condition ensures that the reward slopes associated with parameters in $\mathcal C_k$ differ by at
most a constant factor. This allows the estimated slope to be used as a weight.

Finally, the algorithm updates $V_t$ and $H_t$ by adding $\phi(s_t,a_t)\phi(s_t,a_t)^\top$ and
$w_t\phi(s_t,a_t)\phi(s_t,a_t)^\top$, respectively. The unweighted matrix $V_t$ records all selected features,
whereas $H_t$ incorporates local slope information only when the prediction width is sufficiently small. In
particular, when $w_t=0$, $H_t$ remains unchanged, whereas $V_t$ continues to accumulate information and may
still trigger a policy update. While Algorithm~\ref{al1} monitors only $V_t$, \galgo{} starts a new episode
when either $\det(V_t)\geq 2\det(\overline V_k)$ or $\det(H_t)\geq 2\det(\overline H_k)$. At the start of a
new episode, the algorithm updates the confidence set, estimator, and policy, and resets both reference
matrices.

\section{Experimental Details}\label{app_exp}

This appendix provides additional details of our experiments. We first describe the planning algorithm used to
approximate the value function $J_k(s)$ under small-scale settings.

\subsection{Discounted Howard Policy Iteration}\label{app_howard_policy_iteration}

At the start of episode $k$, we use discounted Howard policy iteration to approximate $J_k(s)$ in
Equation~(\ref{align_calculate_value_function}). The subroutine takes the frozen optimistic reward function
$u=u_k$, discount factor $\gamma\in(0,1)$, and iteration budget $I_{\max}$. Since the memory transition $f$ is
known and deterministic, planning requires no additional interaction with the environment. The initial policy
is myopic: it selects an action that maximizes the immediate optimistic reward,
\begin{align}\label{align_howard_policy_initialize}
    \pi_0(s)\in\arg\max_{a\in\mathcal A}u(s,a).
\end{align}

\begin{algorithm}[htbp]
\caption{Discounted Howard Policy Iteration at Episode $k$}
\label{alg:discounted_policy_iteration}
\begin{algorithmic}[1]
    \STATE \textbf{Input:} Reward function $u$, discount factor $\gamma\in(0,1)$, and iteration budget
    $I_{\max}$.

    \STATE \textbf{Initialize:} Set $\pi_0$ according to \Cref{align_howard_policy_initialize}.

    \FOR{$i=0,\ldots,I_{\max}-1$}
        \STATE Update $\pi_{i+1}$ according to \Cref{align_howard_policy_update} with $\widetilde Q_i(s,a)$
        and $\widetilde J_i(s)$ defined as in \Cref{align_howard_calculate_value}.
    \ENDFOR

    \STATE \textbf{return} $\widetilde J_{I_{\max}}$ defined as in \Cref{align_howard_calculate_value}.
\end{algorithmic}
\end{algorithm}

At iteration $i$, the algorithm updates the policy as
\begin{align}\label{align_howard_policy_update}
    \pi_{i+1}(s) \in \arg\max_{a\in\mathcal A}\widetilde Q_i(s,a),
\end{align}
where $\widetilde Q_i(s,a)$ and $\widetilde J_i(s)$ denote the action value and the discounted value of policy
$\pi_i$, respectively, defined as
\begin{align}
    & \widetilde Q_i(s,a) = u(s,a)+\gamma\widetilde J_i(f(s,a)), \quad \widetilde J_i(s) = u(s,\pi_i(s))+
    \gamma\widetilde J_i(f(s,\pi_i(s))).
    \label{align_howard_calculate_value}
\end{align}
Since $\gamma\in(0,1)$, $\widetilde J_i$ is uniquely determined by this Bellman equation. Unlike the myopic
initialization, this update accounts for both the immediate optimistic reward and the value of the next memory
state. The algorithm performs $I_{\max}$ policy improvements. It returns $\widetilde J_{I_{\max}}$ as an
approximation to $J_k$, which is then used in the action-selection rules.

\begin{theorem}\label{th_howard_regret}
Under the conditions of Theorem~\ref{th_upper_bound}, consider \algo{} with $J_k$ in
Equation~(\ref{align_maximize_action_value_function}) replaced at each episode by $\widetilde J_{I_{\max}}$
returned by Algorithm~\ref{alg:discounted_policy_iteration}. Let $\pi$ denote the resulting policy. Then,
\begin{align*}
    \operatorname{Reg}_T(\pi) & \leq \left( 4d\log\left(1+\frac{TR^2}{\lambda d}\right)+6 \right)RSM +
    4RSd\log\left(1+\frac{TR^2}{\lambda d}\right)
    \\& + 4\sqrt{ Td\log\left(1+\frac{TR^2}{\lambda d}\right) } \left[ \sigma\sqrt{
    d\log\left(1+\frac{TR^2}{\lambda d}\right) +2\log\left(\frac{1}{\delta}\right) } +\sqrt{\lambda}S \right]
    \\& +2RST^2\left(1-\frac1T\right)^{I_{\max}+1},
\end{align*}
holds with probability at least $1-\delta$. Furthermore, if $I_{\max}\geq\lceil 2T\log T\rceil$, the
additional term is at most $2RS$, and hence
\[
    \operatorname{Reg}_T(\pi)= \widetilde O\left(dRS(M+1)+\sigma d\sqrt T\right).
\]
\end{theorem}

\subsection{Synthetic Experimental Setup}\label{app:synthetic_setup}

\paragraph{Data processing and parameter selections.} For the instance construction, we use the coordinate
action set $\mathcal{A}=\{e_1,\ldots,e_{10}\}$ with a fixed ordering. For a history
$s=(a_{t-M},\ldots,a_{t-1})$, the memory map is
\[
    A(s)=
    \begin{cases}
    R e_i e_i^\top, & \text{if } a_{t-M}=\cdots=a_{t-1}=e_i,\\
    0_{d\times d}, & \text{otherwise}.
    \end{cases}
\]
Thus, after switching to an action, the learner must repeat it for $M$ rounds before receiving a nonzero mean
reward. The parameter is $\theta^\star=Su$, where $u$ is a fixed nonnegative unit vector.

In the implementation, \algo{} uses regularization $\lambda=0.625$ and confidence parameter $\delta=0.05$,
together with the discount factor specified in the main text. Optimistic rewards are clipped to $[-RS,RS]$.
For the block baselines, the two OFUL variants use their respective episode-length rules; when the nominal
informative length of OFUL (original) is nonpositive, we set it to $1$.

\paragraph{Comparator.} We use ten balanced initial histories constructed from cyclic shifts of a repeated,
fixed random permutation of the actions. Each history contains all actions equally often, and each action
appears last in one of the ten histories. Results are averaged over these histories within each noise seed.
All methods share the same initial histories and paired noise sequences, and the parameter rules are unchanged
across the four sweeps. The comparator $\mathrm{OPT}_T(\theta^\star)$ is thus the exact finite-horizon optimum
from the same initial history. Let $\ell_i$ denote the length of the trailing run of $e_i$ in this history,
with $\ell_i=0$ if the history does not end in $e_i$. Then
\[
    \mathrm{OPT}_T(\theta^\star) = RS \max_{i\in[d]} u_i (T-M+\ell_i)_+,
\]
where $(x)_+=\max\{x,0\}$. Each term is the cumulative mean reward obtained by repeating $e_i$ throughout the
horizon, including the initial zero-reward rounds. These constant action sequences suffice: active rewards are
constant and nonnegative, and each switch requires $M$ zero-reward rounds, so delaying a switch or switching
repeatedly cannot improve upon either staying with the initial trailing action or switching immediately to
another action.

Moreover, the same structure also allows us to compute $J_k$ exactly by comparing the discounted returns of
repeatedly selecting each action under $u_k$. We only need to track the last action and how many consecutive
times it has been selected, capped at $M$, giving $O(dM)$ effective states on which the value function and its
greedy policy are computed exactly.

\subsection{Additional Synthetic Experiments}

To complement the synthetic experiments in the main text, we consider an additive memory matrix of the form
studied by \citet{clerici2023linear, anngeneralized}. In this setting, the preceding actions contribute
jointly to the memory matrix, and the learner selects actions from a continuous action set.

\paragraph{Data processing and parameter selections.} We use $d=10$,
$\mathcal{A}=\{a\in\mathbb{R}^d:\|a\|_2\leq 1\}$, and independent Gaussian noise with standard deviation
$\sigma=0.25$. For a history $s_t=(a_{t-M},\ldots,a_{t-1})$, the memory map is
\[
    A(s_t)=\frac{R}{M+1} \left(I_d+\sum_{j=1}^{M}a_{t-j}a_{t-j}^{\top}\right).
\]
The normalization ensures $\|A(s_t)\|_{\mathrm{op}}\leq R$, allowing us to vary the memory length and
operator-norm bound separately. We set $\theta^\star=Su$, using the same fixed nonnegative unit direction $u$
as in the synthetic experiments in the main text. All methods start from zero initial memory and share paired
noise sequences. Unless varied, $T=10^6$, $M=100$, and $R=S=1$. We record a continuous trajectory up to
$T=10^6$ and separately vary $M\in\{100,200,\ldots,1000\}$ and $R,S\in\{10,20,\ldots,100\}$, one parameter at
a time. RSM-LinUCB uses $\lambda=0.625$, $\delta=0.05$, and discount factor $\gamma=0.999$, with optimistic
rewards clipped to $[-RS,RS]$. We compare against the same four baselines as in the synthetic experiments in
the main text. Results average 24 seeds; shading denotes one sample standard deviation.

\paragraph{Comparator.} We measure cumulative pseudo-regret against the exact finite-horizon optimum from the
same zero initial memory. At round $t$, the mean reward is at most $RS(1+\min\{t-1,M\})/(M+1)$, and the
constant action sequence $a_t=u$ attains this bound at every round. Thus,
\[
    \mathrm{OPT}_T(\theta^\star) =\frac{RS}{M+1}\sum_{t=1}^{T} \bigl(1+\min\{t-1,M\}\bigr),
\]
which equals $RS(T-M/2)$ for $T\geq M$, including the initial memory build-up. For the implementation of
RSM-LinUCB, we follow the gradient-based optimization approach of \citet{clerici2023linear, anngeneralized}
and approximate optimistic planning by projected gradient ascent over periodic action sequences of length $M$.
We evaluate each sequence's discounted return from the current memory, including its infinite periodic tail,
and use at most 60 iterations per initialization, with four initializations and a warm start from the previous
sequence. The selected sequence is repeated until the next determinant trigger.

\paragraph{Experimental Results.} As shown in Figure~\ref{fig:additional_synthetic}, RSM-LinUCB achieves the
lowest mean regret at every tested value of $M$, $R$, and $S$. At $T=10^6$, its mean regret is approximately
$52.7\%$ lower than that of the best baseline. These results complement the main experiments by demonstrating
the advantage of our memory-aware planning approach under additive memory dynamics and a continuous action
set.

\begin{figure}[htbp]
    \centering
    \includegraphics[width=\linewidth]{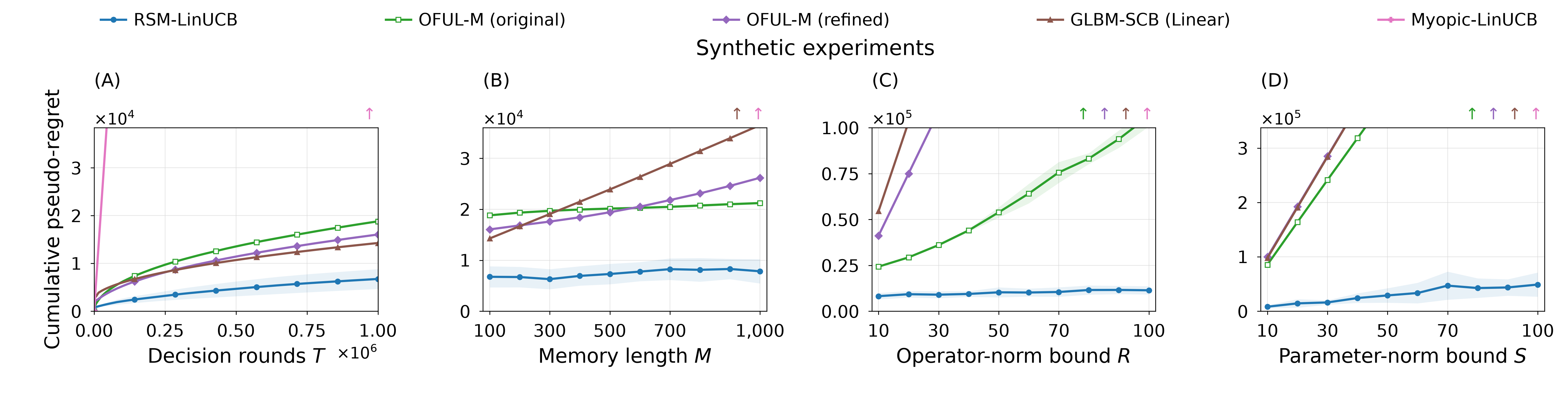}
    \caption{Cumulative pseudo-regret for additional synthetic experiments with additive memory. Shading
    denotes one sample standard deviation. Some baseline curves exceed the displayed range.}
    \label{fig:additional_synthetic}
\end{figure}

\subsection{KV Cache Experimental Setup}\label{app:kv_cache_setup}

\paragraph{Data processing and parameter selections.}

We partition the Mooncake conversation trace~\citep{qin2025mooncake} into construction, calibration, and
evaluation sets using a fixed, content-independent hash of its 10-second timestamp bins. Entire bins are
assigned in proportions $40\%/20\%/40\%$, yielding 4,978, 2,230, and 4,823 requests, respectively. We allocate
a larger construction share than in a preliminary $20\%/20\%/60\%$ split to identify more reusable prefixes
and obtain a richer action catalogue. This is quite useful when studying longer memory horizons: a small
catalogue can become largely covered by the retained actions, leaving little scope for comparing different
cache management policies. The split determines the available data, whereas the tested memory lengths are
chosen separately from the retention timescale. Using only construction requests, we retain prefixes observed
at least twice and project each request onto its deepest retained prefix. Distinct projected endpoints define
candidate refresh actions. Shared blocks count only once toward capacity and reward. We then contract
nonbranching chains whose blocks always have identical availability, preserving the permitted endpoints and
branch points. The resulting catalogue contains $K=855$ actions and $d=864$ informative segment coordinates.
The catalogue and demand representation remain fixed across all tested values of $M$; blocks outside the
retained catalogue contribute no reusable-token reward.

Rewards are divided by a fixed token scale $Z=131{,}072$ and averaged over 36 independently sampled requests
per round. For the reported experiments, we use the conservative request-level range $[0,1]$, giving the
sub-Gaussian noise parameter $\sigma=1/(2\sqrt{36})=1/12$. The parameter norm bound is set to
$S=1.25\|\widehat{\theta}_{\mathrm{cal}}\|_2\approx0.1452$, where $\widehat{\theta}_{\mathrm{cal}}$ is the
calibration mean expressed in the scaled feature coordinates. Our method uses
$\lambda=d\sigma^2/S^2\approx284.72$ and confidence level $10^{-5}$. The theoretical guarantee requires
$\|\theta^\star\|_2\leq S$; we verify this condition for the fixed empirical evaluation distribution without
exposing its demand weights to the learner. The baselines (OFUL-M, GLBM-SCB) retain their respective
block-length rules, confidence interval, and estimation updates. Nonpositive prescribed informative lengths
$L$ are replaced by one. GLBM-SCB uses its ordered online mirror-descent updates, with the implemented
initialization $\lambda_{\mathrm{GLBM}}=L/d$. For scalability, the baseline methods use the same directional
tree-planning machinery described below. Myopic LinUCB maximizes the confidence score for current state
without planning future cache effects.

\paragraph{Simulation protocol.} Each round represents a 10-second cache control epoch. At the start of each
round, the controller refreshes one catalogue prefix, making its blocks available for current and the
following $M$ rounds. Refreshing an already resident block renews its retention. The $M=180$ setting therefore
represents approximately 30 minutes of retention, motivated by Mooncake Store's documented soft-pin
TTL~\citep{mooncakeStore}. To keep cache management lightweight, reduce communication overhead, and limit
exposure of individual access patterns, we model a management controller separate from the request-serving
engine, motivated by aggregate cache monitoring and privacy concerns in KV-cache
sharing~\citep{lmcacheMetrics,wu2025know}. The controller makes refresh decisions at the epoch level, while
the engine handles per-request matching to KV blocks and exports only epoch-level reuse statistics rather than
prompt contents, request prefix identifiers, or tensor payloads. This interface is motivated by the aggregate
prefix-cache query and hit counters available in serving systems such as vLLM~\citep{vllmMetrics}. The
controller knows the fixed action catalogue and its own refresh history, which suffice to compute the cache
feature; aggregate rewards then allow it to estimate unknown demand without identifying each arriving
request's KV blocks. This separation avoids maintaining request-level access logs at the controller while
retaining the feedback needed for learning.

Let $p_i\in\{0,1\}^d$ indicate the segments included in refresh action $i$, and let $g(s_t)\in\{0,1\}^d$
indicate those still resident from the preceding $M$ rounds. After action $i_t$, the available segments are
indicated by $w_t=g(s_t)+(1-g(s_t))\odot p_{i_t}$: previously resident segments remain available, and the
current action adds the missing segments. We use the scaled cache feature $x_t=w_t/\sqrt d$. Both $w_t$ and
$x_t$ are known to the controller because they depend only on its refresh history and the fixed action
catalogue. For each request $Q$, let $z_j(Q)$ be the number of matching prefix tokens in retained segment $j$.
The number of reusable tokens is then $w_t^\top z(Q)$. Each round draws 36 requests and returns their
normalized average reuse, $$Y_t=(36Z)^{-1}\sum_{j=1}^{36}w_t^\top z(Q_{t,j}),$$ where $Z$ is the fixed
normalization constant. Defining $\theta^\star=(\sqrt d/Z)\mathbb E_{\mathrm{eval}}[z(Q)]$ gives
$Y_t=x_t^\top\theta^\star+\varepsilon_t$, where $\varepsilon_t$ has zero mean conditional on the past and the
chosen action. The controller observes only $Y_t$; the individual demand vectors are used by the simulator to
generate this feedback.

To express this model in the paper's \model{} notation, we add one coordinate. Let $h=\max_i\|p_i\|_2^2$ and
define
\[
    a_i=\frac{1}{\sqrt{1+h}}
    \begin{pmatrix}1\\p_i\end{pmatrix},
    \quad A(s)=\sqrt{\frac{1+h}{d}}
    \begin{pmatrix}
    0 & g(s)^\top\\
    0 & \operatorname{diag}(1-g(s))
    \end{pmatrix}.
\]
These definitions ensure $\|a_i\|_2\leq1$ and $A(s_t)^\top a_{i_t}=(0,x_t)$. With augmented parameter
$\bar\theta^\star=(0,\theta^\star)$, the conditional mean reward is exactly
$\langle a_{i_t},A(s_t)\bar\theta^\star\rangle$. The added feature coordinate is always zero and is omitted
from regression, leaving $d$ informative coordinates.

\paragraph{Comparator.} The comparator $\operatorname{OPT}_T(\theta^\star)$ is the maximum cumulative mean
reward over all $T$-round refresh sequences in the fixed catalogue, starting from the same empty cache as the
learner. We compute it offline using mean token demands under the empirical evaluation distribution.

Let $q=M+1$, and let $F_j$ denote the largest per-round mean reward obtainable from the union of at most $j$
catalogue paths. These values are computed exactly by tree dynamic programming using the mean demand weights.
To account for the initial cache build-up, we construct an ordered list by repeatedly adding the path with the
largest additional demand coverage. For the reported catalogue, we verify against the exact dynamic program
that the first $j$ paths attain $F_j$ for every $j\leq q$, using integer token counts before normalization.
Refreshing these paths in order and then cycling through them attains $F_{\min\{t,q\}}$ at every round. No
refresh sequence can exceed this value at round $t$, since its cache contains the union of at most
$\min\{t,q\}$ paths. Therefore, we set \( \operatorname{OPT}_T(\theta^\star) = \sum_{t=1}^{T}
F_{\min\{t,M+1\}} \) to denote the optimal cumulative reward over $T$ rounds in the experiment. Given the mean
demand weights, this offline computation takes $O(dq^3+Kdq)$ time, including the tree-DP calls for all
$j\leq q$ and verification of the ordered paths.

\paragraph{Accelerated planning.} In the KV cache experiment, we utilize the special prefix-cache structure to
obtain a computationally efficient approximation to the optimistic planning step in \algo{}. Specifically, we
approximate the ellipsoidal confidence bonus in $u_k$ using a finite set of Gaussian directions to use tree
dynamic programming for the linear objectives:
\begin{align}\label{align_new_confidence_bound}
    \widetilde u_k(x) = \operatorname{clip}_{[0,B]} \left( x^\top\overline{\theta}_k +
    \frac{\overline{\beta}_k}{c} \max_{1\leq j\leq N} \left|g_{k,j}^\top L_k^\top x\right| \right),
\end{align}
where $x$ is the informative cache feature, $L_kL_k^\top=\overline V_k^{-1}$, $g_{k,j}\sim\mathcal N(0,I_d)$
are independent directions drawn at each planning episode, and $B=RS$ is the known expected reward bound. We
use $c=1.5$ and choose $N=O(\log(T/\delta_{\mathrm{plan}}))$ to control the additional failure probability.

The resulting score is the clipped maximum of $2N$ linear objectives, each assigning an additive weight to
every tree segment. For each objective, a tree-knapsack dynamic program finds the best union of at most
$q=M+1$ prefix paths exactly, counting shared segments only once. We then approximate the planning policy by
selecting the union with the highest score and cyclically refreshing its constituent paths until the next
determinant trigger. This implementation leaves the estimator and update schedule unchanged, and the resident
cache equals the selected union after at most $M$ rounds. Each tree optimization takes $O(dq^2)$ time.
Including matrix factorization and direction transformations, each replanning call takes $O(d^3+Nd^2+Ndq^2)$
time and $O(d^2+Ndq)$ planning memory in the batched implementation.

As in polyhedral confidence set methods~\citep{dani2008stochastic} and randomized linear
bandits~\citep{pmlr-v54-abeille17a}, our approximation introduces only an additional $O(\sqrt d)$ factor in
the estimation error (the second term in Theorem~\ref{th_upper_bound}), up to logarithmic terms. The ridge
estimator and its confidence guarantee remain unchanged, and the additional cost only comes from the
confidence bound used in the optimistic reward estimate $\widetilde u_k$ in
Equation~(\ref{align_new_confidence_bound}). Let $x^\star$ denote an optimal steady cache feature. Each
Gaussian direction has a constant probability of producing a confidence bound at least as large as the
original one at $x^\star$. Taking the maximum over the prescribed number of independent directions therefore
ensures $\widetilde u_k(x^\star)\geq (x^\star)^\top\theta^\star$ across episodes with high probability.
Furthermore, Gaussian concentration also controls the enlargement of the confidence bound uniformly over cache
features, including the feature $x_k$ selected by maximizing $\widetilde u_k$. Together with optimism at
$x^\star$, this gives
\[
    (x^\star-x_k)^\top\theta^\star \leq \widetilde u_k(x_k)-x_k^\top\theta^\star \leq \widetilde O(\sqrt d)\,
    \overline\beta_k \|x_k\|_{\overline V_k^{-1}}.
\]
Thus, the reward gap is still controlled by the original confidence width, enlarged by only a $\sqrt d$
factor. Summing these widths introduces only an additional $\widetilde O(\sqrt d)$ factor in the estimation
error. The memory transients under cyclic refreshes remain bounded by $BM$ per episode, so the additive memory
dependence is preserved.

\subsection{Semantic Cache Experimental Setup}\label{app+semantic_cache}

\paragraph{Data processing and parameter selections.} We use the released GTE embeddings and GPT-4o-mini
responses from vCache/SemBenchmarkLmArena. Exact-query deduplication leaves 56,438 queries, which we randomly
divide into disjoint 20\%/20\%/60\% construction, calibration, and evaluation splits. After excluding
construction entries larger than 64 KiB, we partition the remaining 11,276 entries into $K=64$ disjoint cache
bundles by recursively splitting along the embedding coordinate with the largest variance at its median,
giving $d=65$ features. Each query is assigned to its nearest construction entry using cosine similarity, with
routing fixed independently of which bundles are resident. Reuse is counted as successful only if the
similarity is at least $0.9$ and the two queries share the benchmark's semantic-equivalence label. These
labels are available only to the evaluator and do not certify factual correctness. Bundle sizes include the
query and response text and float32 embeddings. For the 64 MiB payload budget $C$, we require $(M+1)\max_i
\mathrm{size}(i)\le C$ for the memory length $M$, ensuring feasibility for every refresh history. The budget
excludes search index and runtime overheads. For GLBM-SCB, we use blocks of length $M+L$, where
$L=\max\{1,\lfloor\sqrt{MT}\rfloor-M\}$ as used in their paper~\citep{anngeneralized}.

\paragraph{Cache dynamics and generalized rewards.} An action refreshes one bundle, which remains available in
the current round and the following $M$ rounds. Read hits do not extend retention, and the simulator performs
no additional admissions or evictions. Let $g(s)\in\{0,1\}^K$ indicate the bundles refreshed in the preceding
$M$ rounds. For $1\le M<K$, using the standard basis $e_0,\ldots,e_K$ of $\mathbb R^{K+1}$, we define action
set $\mathcal A=\{a_1,\ldots,a_K\}$, with action $a_i=\sqrt{\frac{M}{M+1}}e_0+\frac{1}{\sqrt{M+1}}e_i$, and
memory matrix
\[
    A(s)^\top=
    \begin{pmatrix}
    0 & 0_{1\times K}\\
    \sqrt{\frac{M+1}{M}}g(s) & \sqrt{M+1}\operatorname{diag}(\mathbf 1-g(s))
    \end{pmatrix}.
\]
Thus, every action has unit norm, $\|A(s)\|_{\mathrm{op}}\le R=\sqrt{M+1}$, and
\[
    \phi(s,a_i) = A(s)^\top a_i =
    \begin{pmatrix}
    0\\
    g(s)+\operatorname{diag}(\mathbf 1-g(s))e_i^{(K)}
    \end{pmatrix},
\]
where $e_i^{(K)}$ is the $i$-th standard basis vector of $\mathbb R^K$. The nonzero coordinates therefore
record exactly which bundles are available after the refresh, giving $d=K+1=65$.

The calibration split fixes the bundle arrival proportions $\ell_i$, while the evaluation split determines the
unknown conditional success probabilities $p_i$. Simulated queries first select a bundle according to $\ell_i$
and then sample a held-out query assigned to that bundle. Accordingly,
\[
    \theta^\star=(0,\ell_1p_1,\ldots,\ell_Kp_K)^\top,
\]
and thus $S=\|(\ell_1,\ldots,\ell_K)\|_2\approx0.1256$. The learner observes only the resulting service time
feedback. With unit rate Poisson arrivals and independently sampled queries, successful cache service has rate
$z=\phi(s,a)^\top\theta^\star$, so interruptions have rate $1-z$. Keeping the cache fixed during each service
cycle gives $Y_t\sim\operatorname{Exp}(1-z)$, with $\mu(z)=1/(1-z)$ and $v=1$. All tested settings satisfy
$RS<1$, with $M$ counting completed service cycles.

\paragraph{Comparator.} The comparator $\mathrm{OPT}_T(\theta^\star)$ is the maximum cumulative mean reward
over all $T$-round refresh sequences in the fixed catalogue, starting from the same empty cache as the learner
and following the same retention rule. We compute it offline using the simulator's true bundle contributions
$\ell_i p_i$, which are not revealed to the learner. Let $q=\min\{K,M+1\}$. An optimal sequence refreshes the
$q$ highest-contribution bundles in decreasing order and then cycles through them. At round $t$, any feasible
sequence can keep at most $\min\{t,q\}$ bundles resident, and this sequence retains exactly the bundles with
the largest contributions. Since $\mu$ is increasing, it therefore maximizes the mean reward at every round.
Summing these rewards gives the exact finite-horizon optimum, including the initial cache build-up. Sorting
the contributions and computing prefix sums takes $O(K\log K+q)$ time, without enumerating refresh sequences.

\paragraph{Accelerated planning.} We exploit the disjoint-bundle structure to efficiently approximate the
optimistic planning step in Algorithm~\ref{al2}. Specifically, we use a finite-vertex outer approximation for
the confidence maximization in $u_k$ and implement the resulting cache policy through cyclic refreshes.
Following \citet{dani2008stochastic}, we approximate the optimistic reward $u_k(s,a)$ as
\[
    \widetilde u_k(s,a) = \mu\left( \operatorname{clip}_{[-RS,RS]} \left( \phi(s,a)^{\top}\overline\theta_k +
    \sqrt d\max_{0\leq j\leq K} \left|\phi(s,a)^\top F_ke_j\right| \right) \right),
\]
where $F_k=\sqrt{8evc_T(\delta)+4\lambda_2S^2}\,\overline H_k^{-1/2}$ captures parameter uncertainty, $d=K+1$,
and $e_j$ denotes the $j$-th coordinate vector.

This approximation reduces optimistic cache optimization to $2d$ linear objectives. For each objective, the
disjoint-bundle structure allows us to find its best $q$-bundle cache by selecting the $q=\min\{K,M+1\}$
largest bundle coordinates. Comparing these candidates solves the approximate optimistic optimization over
$q$-bundle caches exactly. We then cyclically refresh the selected bundles until the next determinant trigger,
establishing the selected cache after at most $M$ initial rounds. Each replanning call takes $O(d^3)$ time for
matrix factorization and $O(dK\log K)$ time for cache optimization, giving a total complexity of
$O(d^3+dK\log K)$.

As in the polyhedral confidence approximation of~\citet{dani2008stochastic} (Theorem 2), our implementation
also introduces only an additional $O(\sqrt d)$ factor in the estimation error (the second term in
Theorem~\ref{thm:glbm_upper}). The finite-vertex approximation contains the same $\overline H_k$-based
confidence ellipsoid, so the resulting reward estimate remains optimistic, while its prediction confidence
bound increases by at most $\sqrt d$ relative to that ellipsoid. Since the approximate cache optimization is
solved exactly, the reward gap is controlled by this enlarged confidence bound without further error from
numerical optimization. For the reciprocal link $\mu(z)=1/(1-z)$, the relation
\[
    \mu(z')-\mu(z) \leq \sqrt{U_\mu\mu'(z)}\,(z'-z),
\]
for $-RS\leq z\leq z'\leq RS$ allows the prediction error to be translated into a mean-reward difference using
the local slope. The estimation error argument in Theorem~\ref{thm:glbm_upper} therefore applies with only the
additional $O(\sqrt d)$ factor. Importantly, curvature weighting still uses the original confidence-width
bounds rather than the enlarged planning bound, so the curvature-induced error does not acquire this factor.
The determinant triggers are unchanged, and the initial cache transitions still contribute at most
$M\Delta_\mu$ per episode. Thus, the approximation enlarges only the estimation error term, while preserving
the original orders of the memory-induced and curvature-induced errors.

\section{Detailed Proofs}

\begin{proof}[Proof of Lemma~\ref{lemma:value_span}]
Fix any two memory states $s,s'\in\mathcal S$ and any action sequence $(a_0,a_1,\ldots)$. Let $(s_j)_{j\geq0}$
and $(s_j')_{j\geq0}$ denote the resulting state trajectories, where $s_0=s$, $s_0'=s'$, and
\[
    s_{j+1}=f(s_j,a_j), \quad s_{j+1}'=f(s_j',a_j).
\]
Applying the same sequence of $M$ actions causes the two memory states to coincide. Therefore, $s_j=s_j'$ for
every $j\geq M$, and the rewards along the two trajectories may differ only during the first $M$ rounds. Since
every reward lies in $[-RS,RS]$, we obtain
\begin{align*}
    \sum_{j=0}^{\infty}\gamma^j r_q(s_j,a_j) &= \sum_{j=0}^{\infty}\gamma^j r_q(s_j',a_j) +
    \sum_{j=0}^{M-1}\gamma^j \left[r_q(s_j,a_j)-r_q(s_j',a_j)\right] \\
    &\leq J_q(s')+2RS\sum_{j=0}^{M-1}\gamma^j.
\end{align*}
Taking the supremum over all action sequences on the left-hand side gives
\[
    J_q(s)-J_q(s') \leq 2RS\sum_{j=0}^{M-1}\gamma^j \leq 2RSM.
\]
Since $s$ and $s'$ are arbitrary, it follows that
\[
    \operatorname{span}(J_q) = \sup_{s\in\mathcal S}J_q(s) - \inf_{s\in\mathcal S}J_q(s) \leq 2RSM.
\]
This proves the claimed bound.
\end{proof}

\begin{proof}[Proof of Lemma~\ref{lemma:regret_decomposition}]
Recall that
\[
    \beta_t = \sigma \sqrt{ d\log\left(1+\frac{tR^2}{\lambda d}\right) +2\log\left(\frac{1}{\delta}\right) }
    +\sqrt{\lambda}S, \quad \widehat\theta_t=V_t^{-1}b_t.
\]
By the standard self-normalized concentration inequality, the event
\[
    \mathcal E \coloneqq \bigcap_{t=0}^{T} \left\{ \mnorm{\widehat\theta_t-\theta^\star}{V_t} \leq \beta_t
    \right\}
\]
satisfies $\mathbb P(\mathcal E)\geq1-\delta$. Throughout the proof, we work on the event $\mathcal E$. For
every episode $k$ and state--action pair $(s,a)$, the Cauchy--Schwarz inequality gives
\[
    \left| \phi(s,a)^\top \left(\overline\theta_k-\theta^\star\right) \right| \leq \overline\beta_k
    \mnorm{\phi(s,a)}{\overline V_k^{-1}}.
\]
Since $r(s,a)\in[-RS,RS]$, clipping the optimistic estimate to $[-RS,RS]$ preserves optimism. Therefore,
\(u_k(s,a)\geq r(s,a)\) holds for every $k$ and $(s,a)$. For any reward function
$r_q:\mathcal S\times\mathcal A\to[-RS,RS]$, let $J_q$ denote its optimal discounted value function and define
\[
    c_q\coloneqq\inf_{s\in\mathcal S}J_q(s).
\]
By Lemma~\ref{lemma:value_span},
\[
    J_q(s)-c_q\in[0,2RSM]
\]
holds for every memory state $s\in\mathcal S$. After centering the Bellman optimality equation by $c_q$, we
have
\[
    (1-\gamma)c_q+J_q(s)-c_q = \sup_{a\in\mathcal A} \left\{ r_q(s,a) +\gamma\bigl(J_q(f(s,a))-c_q\bigr)
    \right\}.
\]
Now take $r_q=r$. Consider any comparator action sequence $(a_1^\star,\ldots,a_T^\star)$ in
Equation~(\ref{align_opt}), and let $(s_t^\star)_{t=1}^{T+1}$ be its induced state trajectory. Evaluating the
Bellman equation at $(s_t^\star,a_t^\star)$ gives
\begin{align*}
    r(s_t^\star,a_t^\star) \leq\; (1-\gamma)c_{\theta^\star} + \left(
    J_{\theta^\star}(s_t^\star)-c_{\theta^\star} \right)- \gamma \left(
    J_{\theta^\star}(s_{t+1}^\star)-c_{\theta^\star} \right).
\end{align*}
Summing over $t\in[T]$, we obtain
\begin{align*}
    & \sum_{t=1}^{T} r(s_t^\star,a_t^\star)
    \\\leq& T(1-\gamma)c_{\theta^\star} + \left( J_{\theta^\star}(s_1^\star)-c_{\theta^\star} \right) +
    (1-\gamma) \sum_{t=2}^{T} \left( J_{\theta^\star}(s_t^\star)-c_{\theta^\star} \right) - \gamma \left(
    J_{\theta^\star}(s_{T+1}^\star)-c_{\theta^\star} \right) \\
    \leq& T(1-\gamma)c_{\theta^\star} + 2RSM\left[1+(T-1)(1-\gamma)\right].
\end{align*}
Since $\gamma=1-1/T$, taking the supremum over comparator action sequences yields
\begin{align}
    \operatorname{OPT}_T(\theta^\star) \leq c_{\theta^\star}+4RSM.
    \label{align_decomposition_comparator}
\end{align}

We now analyze the trajectory generated by \algo{}. Let
\[
    E_k \coloneqq \{\tau_k,\ldots,\tau_{k+1}-1\}
\]
denote the set of rounds in episode $k$, where $\tau_{K+1}=T+1$, and let
\[
    L_k\coloneqq|E_k|=\tau_{k+1}-\tau_k.
\]
Define \( c_k\coloneqq\inf_{s\in\mathcal S}J_k(s)\). Since $u_k(s,a)\in[-RS,RS]$, Lemma~\ref{lemma:value_span}
also gives
\[
    J_k(s)-c_k\in[0,2RSM]
\]
for every $s\in\mathcal S$. By the action-selection rule in Algorithm~\ref{al1} and the Bellman optimality
equation, for every $t\in E_k$,
\[
    J_k(s_t^\pi) = u_k(s_t^\pi,a_t^\pi) +\gamma J_k(s_{t+1}^\pi).
\]
Centering this equality by $c_k$ gives
\[
    u_k(s_t^\pi,a_t^\pi) + \gamma\left(J_k(s_{t+1}^\pi)-c_k\right) = (1-\gamma)c_k+J_k(s_t^\pi)-c_k.
\]
Summing over $t\in E_k$ yields
\begin{align*}
    &L_k(1-\gamma)c_k - \sum_{t\in E_k}u_k(s_t^\pi,a_t^\pi) \\
    =\;& -\left(J_k(s_{\tau_k}^\pi)-c_k\right) - (1-\gamma) \sum_{t=\tau_k+1}^{\tau_{k+1}-1}
    \left(J_k(s_t^\pi)-c_k\right)+ \gamma \left(J_k(s_{\tau_{k+1}}^\pi)-c_k\right) \\
    \leq\;&2RSM.
\end{align*}

On the event $\mathcal E$, optimism implies \(J_k(s)\geq J_{\theta^\star}(s)\) for every $s\in\mathcal S$, and
hence $c_k\geq c_{\theta^\star}$. Combining these observations, we obtain
\begin{align*}
    &L_k(1-\gamma)c_{\theta^\star} - \sum_{t\in E_k} r(s_t^\pi,a_t^\pi) \\
    \leq\;& L_k(1-\gamma)c_k - \sum_{t\in E_k}u_k(s_t^\pi,a_t^\pi)+ \sum_{t\in E_k} \left[
    u_k(s_t^\pi,a_t^\pi) - r(s_t^\pi,a_t^\pi) \right] \\
    \leq\;& 2RSM + \sum_{t\in E_k} \left[ u_k(s_t^\pi,a_t^\pi) - r(s_t^\pi,a_t^\pi) \right].
\end{align*}

Finally, combining Equation~(\ref{align_decomposition_comparator}) with the preceding inequality, summing over
all episodes, and recalling that $\sum_{k=1}^{K}L_k=T$ and $1-\gamma=1/T$, we have
\begin{align*}
    \operatorname{Reg}_T(\pi) =\;& \operatorname{OPT}_T(\theta^\star) - \sum_{t=1}^{T} r(s_t^\pi,a_t^\pi) \\
    \leq\;& c_{\theta^\star}+4RSM - \sum_{k=1}^{K}\frac{L_k}{T}c_{\theta^\star} +2KRSM+ \sum_{t=1}^{T} \left[
    u_{k(t)}(s_t^\pi,a_t^\pi) - r(s_t^\pi,a_t^\pi) \right] \\
    =\;& 2(K+2)RSM + \sum_{t=1}^{T} \left[ u_{k(t)}(s_t^\pi,a_t^\pi) - r(s_t^\pi,a_t^\pi) \right].
\end{align*}
Since $\mathbb P(\mathcal E)\geq1-\delta$, this proves the bound.
\end{proof}

\begin{proof}[Proof of Theorem~\ref{th_upper_bound}]
Let $\mathcal E$ and $\beta_t$ be as defined in the proof of Lemma~\ref{lemma:regret_decomposition}, and
recall that $\mathbb P(\mathcal E)\geq1-\delta$. Throughout the proof, we work on the same event $\mathcal E$
and write $B=RS$.

By Lemma~\ref{lemma:regret_decomposition}, on $\mathcal E$,
\begin{align*}
    \operatorname{Reg}_T(\pi) \leq\;& 2(K+2)BM + \sum_{t=1}^{T} \left[ u_{k(t)}(s_t^\pi,a_t^\pi) -
    r(s_t^\pi,a_t^\pi) \right].
\end{align*}
It therefore remains to bound the number of episodes and the cumulative reward-estimation error. The
confidence bound used in the proof of Lemma~\ref{lemma:regret_decomposition}, together with clipping to
$[-B,B]$, gives
\[
    0 \leq u_k(s,a)-r(s,a) \leq 2\min \left\{ B,\, \overline\beta_k \mnorm{\phi(s,a)}{\overline V_k^{-1}}
    \right\}
\]
for every episode $k$ and state--action pair $(s,a)$. We first bound the total number of episodes. For every
completed episode $k\in[K-1]$, the determinant-doubling condition implies
\[
    \det(\overline V_{k+1}) = \det(V_{\tau_{k+1}-1}) \geq 2\det(\overline V_k).
\]
Since $\overline V_1=V_0=\lambda I_d$, it follows that
\[
    2^{K-1}\det(V_0) \leq \det(\overline V_K) \leq \det(V_T).
\]
Therefore,
\[
    K \leq 1+ \frac{1}{\log 2} \log\left( \frac{\det(V_T)}{\det(\lambda I_d)} \right).
\]
Define
\[
    \Gamma_T \coloneqq \log\left( \frac{\det(V_T)}{\det(\lambda I_d)} \right) \leq d\log\left(
    1+\frac{TR^2}{\lambda d} \right) \eqqcolon \overline\Gamma_T,
\]
where the inequality follows from $\|\phi(s,a)\|_2\leq R$ and the arithmetic--geometric mean inequality
applied to the eigenvalues of $V_T$. Thus,
\begin{align}
    \label{align_number_of_episodes}
    K \leq 1+\frac{\overline\Gamma_T}{\log 2} \leq 1+2\overline\Gamma_T.
\end{align}

We next compare the frozen information matrix with the current one. For any round $t$ in episode $k$, the
episode termination rule ensures that
\[
    \det(V_{t-1}) < 2\det(\overline V_k).
\]
Since $V_{t-1}\succeq\overline V_k$, define
\[
    Q \coloneqq \overline V_k^{-1/2} V_{t-1} \overline V_k^{-1/2}.
\]
Every eigenvalue of $Q$ is at least one, while
\[
    \det(Q) = \frac{\det(V_{t-1})}{\det(\overline V_k)} <2.
\]
It follows that $Q\preceq2I_d$, and therefore
\[
    V_{t-1}\preceq2\overline V_k.
\]
Equivalently, $\overline V_k^{-1}\preceq2V_{t-1}^{-1}$, which implies
\begin{align}
    \label{align_episode_norm_comparison}
    \mnorm{\phi(s_t^\pi,a_t^\pi)}{\overline V_k^{-1}} \leq \sqrt{2}\,
    \mnorm{\phi(s_t^\pi,a_t^\pi)}{V_{t-1}^{-1}}.
\end{align}

Combining Lemma~\ref{lemma:regret_decomposition}, the reward-estimation error bound,
Equation~(\ref{align_number_of_episodes}) and (\ref{align_episode_norm_comparison}), we obtain
\begin{align}
    \operatorname{Reg}_T(\pi) \leq\;& 2(K+2)BM + 2\sum_{t=1}^{T} \min \left\{ B,\,
    \sqrt{2}\,\overline\beta_{k(t)} \mnorm{\phi(s_t^\pi,a_t^\pi)}{V_{t-1}^{-1}} \right\} \notag\\
    \leq& \left( 4d\log\left(1+\frac{TR^2}{\lambda d}\right)+6 \right)BM+ 2\sum_{t=1}^{T} \min \left\{ B,\,
    \sqrt{2}\,\overline\beta_{k(t)} \mnorm{\phi(s_t^\pi,a_t^\pi)}{V_{t-1}^{-1}} \right\}.
    \label{align_analysis_of_regret}
\end{align}

It remains to control the cumulative confidence widths. By the elliptical potential lemma
\citep[Lemma~11]{AbbasiYadkori2011ImprovedAF},
\begin{align}
    \label{align_elliptical_potential}
    \sum_{t=1}^{T} \min \left\{ 1,\, \mnorm{\phi(s_t^\pi,a_t^\pi)}{V_{t-1}^{-1}}^2 \right\} \leq 2d\log\left(
    1+\frac{TR^2}{\lambda d} \right) = 2\overline\Gamma_T.
\end{align}
For notational convenience, let
\[
    q_t \coloneqq \mnorm{\phi(s_t^\pi,a_t^\pi)}{V_{t-1}^{-1}},
\]
and define
\[
    \mathcal I_{>} \coloneqq \{t\in[T]:q_t>1\}, \quad \mathcal I_{\leq} \coloneqq \{t\in[T]:q_t\leq1\}.
\]
By Equation~(\ref{align_elliptical_potential}), $|\mathcal I_{>}|\leq2\overline\Gamma_T$. Moreover, since
$\overline\beta_{k(t)}\leq\beta_T$, the Cauchy--Schwarz inequality gives
\begin{align*}
    \sum_{t=1}^{T} \min \left\{ B,\, \sqrt{2}\,\overline\beta_{k(t)}q_t \right\} &\leq B|\mathcal I_{>}| +
    \sqrt{2}\,\beta_T \sum_{t\in\mathcal I_{\leq}}q_t \\
    &\leq 2B\overline\Gamma_T + \sqrt{2}\,\beta_T \sqrt{ T\sum_{t\in\mathcal I_{\leq}}q_t^2 } \\
    &\leq 2B\overline\Gamma_T + 2\beta_T\sqrt{T\overline\Gamma_T}.
\end{align*}
Substituting this bound into Equation~(\ref{align_analysis_of_regret}) yields
\begin{align*}
    \operatorname{Reg}_T(\pi) & \leq \left( 4d\log\left(1+\frac{TR^2}{\lambda d}\right)+6 \right)BM +
    4Bd\log\left(1+\frac{TR^2}{\lambda d}\right) \\
    &+ 4\sqrt{ Td\log\left(1+\frac{TR^2}{\lambda d}\right) } \left[ \sigma \sqrt{
    d\log\left(1+\frac{TR^2}{\lambda d}\right) +2\log\left(\frac{1}{\delta}\right) } +\sqrt{\lambda}S \right].
\end{align*}
Finally, recalling that $B=RS$, taking $\delta=T^{-2}$, and using $\sqrt{\lambda}S\leq\sigma\sqrt d$, we
obtain
\[
    \operatorname{Reg}_T(\pi) = \widetilde O\left(dRS(M+1)+\sigma d\sqrt T\right).
\]
Since all the bounds hold on the same event $\mathcal E$, the conclusion holds with probability at least
$1-\delta$. This proves the claimed regret bound.
\end{proof}

\begin{proof}[Proof of Theorem~\ref{th_lower_bound}]
First, we consider the case when $M\geq 1$, and establish the memory-dependent lower bound. Let
\(e_1,\ldots,e_d\) denote the standard basis vectors in \(\mathbb R^d\), and define the memory map
\[
    A(s) =
    \begin{cases}
    R e_i e_i^\top, & s=e_i^M\coloneqq(e_i,\ldots,e_i) \text{ for some }i\in[d],\\
    0, & \text{otherwise}.
    \end{cases}
\]
Clearly, \(\|A(s)\|_{\mathrm{op}}\leq R\) for every \(s\). For each \(i\in[d]\), let \( \theta^{(i)}=Se_i,\)
and take the action set to be \(\mathcal A=\mathbb B_2^d.\) All environments, including the null environment,
share the initial memory $s_1=(0_d,\ldots,0_d)\in\mathcal A^M$, which is also used by the comparator. For the
first $M$ rounds, the memory therefore contains at least one zero action, so $s_t\neq e_i^M$ for every
$i\in[d]$ and all mean rewards are zero. The expected reward under parameter \(\theta^{(i)}\) is then
\[
    r_{\theta^{(i)}}(s_t,a_t) =
    \begin{cases}
    RS(a_t)_i, & s_t=e_i^M,\\
    0, & \text{otherwise}.
    \end{cases}
\]
For the environment \(\theta^{(i)}\), an optimal comparator repeatedly selects \(e_i\). Since the first \(M\)
rounds yield zero reward, whereas every subsequent round yields reward \(RS\), we have
\[
    \operatorname{OPT}_T\left(\theta^{(i)}\right) = RS(T-M).
\]

Let \(\operatorname{Reg}_T^{(i)}(\pi)\) denote the regret of policy \(\pi\) under parameter \(\theta^{(i)}\),
and let \(\theta^{(0)}=0\) denote the null environment. We couple all environments by using the same
realization of the learner's internal randomness and the same noise sequence. Consider the trajectory
generated by \(\pi\) in the null environment, and define
\[
    \tau_i = \inf \left\{ t\in[T]: s_t=e_i^M \right\}.
\]
If the set is empty, we set \(\tau_i=\infty\). We further define the truncated hitting time
\[
    \overline\tau_i =
    \begin{cases}
    \tau_i, & \tau_i\leq T,\\
    T+1, & \tau_i=\infty.
    \end{cases}
\]

Before round \(\overline\tau_i\), all conditional reward means are zero under both \(\theta^{(0)}\) and
\(\theta^{(i)}\). Hence, under the above coupling, the two trajectories coincide up to the action selected at
round \(\overline\tau_i\). Since the expected reward in each subsequent round is at most \(RS\), it follows
pathwise that
\[
    \sum_{t=1}^{T} r_{\theta^{(i)}}(s_t^\pi,a_t^\pi) \leq RS\left(T-\overline\tau_i+1\right).
\]
Consequently,
\[
    \operatorname{Reg}_T^{(i)}(\pi) \geq RS\left(\overline\tau_i-(M+1)\right),
\]
and therefore
\[
    \mathbb E\left[ \operatorname{Reg}_T^{(i)}(\pi) \right] \geq RS\, \mathbb E\left[ \overline\tau_i-(M+1)
    \right].
\]

Now consider two distinct indices \(i,j\in[d]\). If \(\tau_i<\tau_j<\infty\), then the \(M\) actions
immediately preceding \(\tau_i\) are all equal to \(e_i\), whereas those immediately preceding \(\tau_j\) are
all equal to \(e_j\). Since \(i\neq j\), these two blocks cannot overlap, and thus \( \tau_j\geq\tau_i+M.\)
Let \( \overline\tau_{(1)} \leq\cdots\leq \overline\tau_{(d)} \) denote the order statistics of
\(\overline\tau_1,\ldots,\overline\tau_d\). Since \(\overline\tau_{(1)}\geq M+1\), the preceding separation
property implies
\[
    \overline\tau_{(q)}-(M+1) \geq \min\left\{ (q-1)M,\, T-M \right\}
\]
for every \(q\in[d]\). Therefore,
\begin{align}
    \label{align_lower_bounding_average_d}
    \frac{1}{d} \sum_{i=1}^{d} \mathbb E\left[ \operatorname{Reg}_T^{(i)}(\pi) \right] &\geq \frac{RS}{d}
    \sum_{q=1}^{d} \min\left\{ (q-1)M,\, T-M \right\} \notag\\
    &= \frac{RS}{d} \sum_{q=0}^{d-1} \min\left\{ qM,\, T-M \right\}.
\end{align}
Under the assumption \(T\geq dM\), we have \(qM\leq T-M\) for every \(q\in\{0,\ldots,d-1\}\). Hence,
\[
    \frac{1}{d} \sum_{i=1}^{d} \mathbb E\left[ \operatorname{Reg}_T^{(i)}(\pi) \right]
    \geq \frac{(d-1)RSM}{2}.
\]
It follows that
\begin{align}
    \label{align_lower_bounding_sup}
    \sup_{\theta^\star,A} \mathbb E\left[ \operatorname{Reg}_T(\pi) \right] \geq \frac{1}{d} \sum_{i=1}^{d}
    \mathbb E\left[ \operatorname{Reg}_T^{(i)}(\pi) \right] \geq \frac{(d-1)RSM}{2}.
\end{align}

For \(M=0\), take \(A=RI_d\). The problem then reduces to a standard linear bandit over \(\mathbb B_2^d\),
with parameter \(R\theta^\star\) satisfying \(\|R\theta^\star\|_2\leq RS\). Applying Proposition~2
of~\citet{zhang2025linear} with their scale parameter set to \(RS/4\), we obtain
\[
    \sup_{\theta^\star,A} \mathbb E\left[\operatorname{Reg}_T(\pi)\right] \geq cdRS
\]
for some universal constant \(c>0\), provided that \(T\geq d\). Combining the two cases, we then obtain
\begin{align}
    \label{align_lower_bounding_sup_mplusone}
    \sup_{\theta^\star,A} \mathbb E\left[\operatorname{Reg}_T(\pi)\right] \geq \Omega\left (dRS(M+1) \right),
\end{align}
for every \(M\geq0\). We next establish the standard linear-bandit lower bound. Let the noise variables be
independent Gaussian random variables,
\[
    \varepsilon_t \sim \mathcal N(0,\sigma^2),
\]
and consider the memory-independent map \( A(s)=RI_d \) for every $s\in\mathcal A^M$. This reduces the problem
to a stochastic linear bandit over the unit Euclidean ball with effective parameter \(R\theta^\star\).
Applying Theorem~24.2 of~\citet{lattimore2020bandit} after rescaling the observations by \(\sigma\), there
exists a parameter \(\theta^\star\) satisfying \( \|\theta^\star\|_2^2 = \frac{\sigma^2d^2}{48TR^2} \) such
that
\[
    \mathbb E\left[ \operatorname{Reg}_T(\pi) \right] \geq \frac{\sigma d\sqrt T}{16\sqrt 3}.
\]
The parameter is admissible whenever \(RS\geq\sigma d/\sqrt{48T}.\) Therefore,
\begin{align}
    \label{align_lower_bounding_sup2}
    \sup_{\theta^\star,A} \mathbb E\left[ \operatorname{Reg}_T(\pi) \right]
    \geq \frac{\sigma d\sqrt T}{16\sqrt 3}.
\end{align}

Combining Equation~(\ref{align_lower_bounding_sup_mplusone}), we conclude that, for every policy \(\pi\),
\begin{align*}
    \sup_{\theta^\star,A} \mathbb E\left[ \operatorname{Reg}_T(\pi) \right] &\geq \Omega\left( dRS(M+1) +
    \sigma d\sqrt{T} \right).
\end{align*}
Taking the infimum over all policies proves Theorem~\ref{th_lower_bound}.
\end{proof}

\begin{proof}[Proof of Theorem~\ref{thm:glbm_upper}]
We work on the event $\mathcal E$ from Lemma~\ref{lem:glbm_likelihood_confidence} which happens with
probability at least $1-\delta$. According to Lemma~\ref{lem:glbm_episodes_regret}, it remains to bound the
reward estimation error.

Consider first a round $t$ with $w_t>0$, and let $k=k(t)$ denote its episode. The definition of $w_t$ and
Equation (\ref{eq:glbm_width_curvature}) in Appendix~\ref{app_lemma} imply that \(
\mu'(\phi_t^\top\theta)\le\mathrm ew_t \) for every $\theta\in\mathcal C_k$. Since $\theta^\star$ and the
maximizer defining $u_k(s_t,a_t)$ belong to the convex set $\mathcal C_k$, the mean value theorem and
Lemma~\ref{lem:glbm_prediction_widths} give that
\begin{align*}
    u_k(s_t,a_t)-r(s_t,a_t) &\le\mathrm e\,w_t\omega_k(\phi_t)\le\sqrt2\,\mathrm e\,\beta_2
    w_t\|\phi_t\|_{H_{t-1}^{-1}}.
\end{align*}
Further denote that $\ell_t=w_t\|\phi_t\|_{H_{t-1}^{-1}}^2$. Applying the elliptical potential lemma (Lemma 11
in~\citet{AbbasiYadkori2011ImprovedAF}) to the vectors $\sqrt{w_t}\phi_t$ yields
\[
    \sum_{t=1}^T\min\{1,\ell_t\} \le2\log\left(\frac{\det H_T}{\det H_0}\right) \le2\Gamma_2.
\]
Therefore, by the Cauchy--Schwarz inequality and $\sum_{t=1}^T w_t\le TU_\mu$,
\begin{align*}
    \sum_{\substack{t:w_t>0\\ \ell_t\le1}} \left[u_{k(t)}(s_t,a_t)-r(s_t,a_t)\right]
    &\le \sqrt2\,\mathrm e\,\beta_2 \sqrt{ \left(\sum_{t=1}^T w_t\right) \left(\sum_{\substack{t:
    \ell_t\le1}}\ell_t\right) }\\
    &\le2\mathrm e\,\beta_2\sqrt{U_\mu T\Gamma_2}.
\end{align*}
The determinant identity $\sum_{t=1}^T\log(1+\ell_t) =\log(\det H_T/\det H_0)$ also implies that at most
$\Gamma_2/\log2$ rounds have $\ell_t>1$. Bounding their estimation error by $\Delta_\mu$ gives
\begin{equation}\label{eq:glbm_positive_weight_error}
    \sum_{\substack{t:w_t>0}} \left[u_{k(t)}(s_t,a_t)-r(s_t,a_t)\right]
    \le 2\mathrm e\,\beta_2\sqrt{U_\mu T\Gamma_2} +\frac{\Delta_\mu\Gamma_2}{\log2}.
\end{equation}

We next consider rounds with $w_t=0$. There are no such rounds if $\rho=0$. Otherwise,
Lemma~\ref{lem:glbm_prediction_widths} gives
\[
    1<\rho\omega_{k(t)}(\phi_t) \le\sqrt2\rho\beta_1\|\phi_t\|_{V_{t-1}^{-1}} \text{ and }
    \|\phi_t\|_{V_{t-1}^{-1}}^2 >\frac{1}{2\rho^2\beta_1^2}.
\]
Using the determinant identity for $V_t$,
\begin{align*}
    \left|\{t:w_t=0\}\right| \le \frac{ \displaystyle\sum_{t=1}^T
    \log\left(1+\|\phi_t\|_{V_{t-1}^{-1}}^2\right)
    }{\log\left(1+1/(2\rho^2\beta_1^2)\right)}\le \frac{\Gamma_1} {\log\left(1+1/(2\rho^2\beta_1^2)\right)}
    \le(1+2\rho^2\beta_1^2)\Gamma_1,
\end{align*}
where the last inequality follows from $\log(1+x)\ge x/(1+x)$ for $x\ge0$. Note that the final bound also
holds when $\rho=0$. Hence,
\begin{equation}\label{eq:glbm_zero_weight_error}
    \sum_{\substack{t:\,w_t=0}} \left[u_{k(t)}(s_t,a_t)-r(s_t,a_t)\right]
    \le \Delta_\mu(1+2\rho^2\beta_1^2)\Gamma_1.
\end{equation}

Combining Lemma~\ref{lem:glbm_episodes_regret} with Equation~(\ref{eq:glbm_positive_weight_error}) and
(\ref{eq:glbm_zero_weight_error}), we obtain
\begin{align}\label{eq:glbm_lemmas_finite_upper}
    & \operatorname{Reg}_T(\textnormal{\galgo{}})\notag
    \\\le&
    \left(3+\frac{\Gamma_1+\Gamma_2}{\log2}\right)M\Delta_\mu +2\mathrm e\,\beta_2\sqrt{U_\mu T\Gamma_2}+
    \frac{\Delta_\mu\Gamma_2}{\log2} +\Delta_\mu(1+2\rho^2\beta_1^2)\Gamma_1
    \\ = & \widetilde O\left( d(M+1)\Delta_\mu +d\sqrt{vU_\mu T} +d^2\kappa v\rho^2\Delta_\mu \right),
\end{align}
which completes the proof.
\end{proof}

\begin{proof}[Proof of Theorem~\ref{thm:glbm_lower}]
By Lemma~\ref{lem:glbm_memory_lower}, there exists an admissible action set and a common initial memory state
for which every policy incurs expected regret at least \(c_1\frac{d(M+1)\Delta_\mu}{1+\log(1+\rho RS)}\) for
some universal constant $c_1>0$. On the other hand, Lemma~\ref{lem:glbm_statistical_lower} gives an admissible
action set for which every policy incurs expected regret at least $c_2d\sqrt{vU_\mu T}$, with a universal
constant $c_2>0$.

For the fixed problem parameters, choose the construction with the larger of the two lower bounds, and we
obtain that
\begin{align*}
    \inf_\pi\sup_{\theta^\star,A} \mathbb E\left[\mathrm{Reg}_T(\pi)\right] &\geq \max\left\{
    c_1\frac{d(M+1)\Delta_\mu}{1+\log(1+\rho RS)}, \ c_2d\sqrt{vU_\mu T} \right\} \\
    &\geq \frac{\min\{c_1,c_2\}}2 \left( \frac{d(M+1)\Delta_\mu}{1+\log(1+\rho RS)} +d\sqrt{vU_\mu T} \right),
\end{align*}
which proves the theorem.
\end{proof}

\begin{proof}[Proof of Theorem~\ref{th_howard_regret}]
We first control the error introduced by Algorithm~\ref{alg:discounted_policy_iteration}. Fix an episode $k$
and use $\widetilde J_i$ to denote the value function computed at iteration $i$ within this episode. Recall
that $J_k$ is the optimal discounted value function under the frozen optimistic reward $u_k$. Since
$\widetilde J_i$ is the value of policy $\pi_i$ under the same reward function, it holds that
\(\widetilde J_i(s)\leq J_k(s)\) for every memory state $s\in\mathcal S$. By the contraction property of
discounted Howard policy iteration (Lemma 2 in~\citet{scherrer2013improved}),
\[
    \left\|J_k-\widetilde J_i\right\|_\infty \leq \gamma^i \left\|J_k-\widetilde J_0\right\|_\infty.
\]
Since $u_k(s,a)\in[-RS,RS]$, both $J_k$ and $\widetilde J_0$ take values in $[-RS/(1-\gamma),RS/(1-\gamma)]$.
Therefore,
\begin{align}
    \left\|J_k-\widetilde J_i\right\|_\infty \leq \frac{2RS\gamma^i}{1-\gamma}.
    \label{align_howard_value_error}
\end{align}

Within episode $k$, the implemented algorithm selects
\[
    a_t^\pi \in \arg\max_{a\in\mathcal A} \left\{ u_k(s_t^\pi,a) + \gamma\widetilde J_{I_{\max}}(f(s_t^\pi,a))
    \right\}.
\]
By the Bellman optimality equation for $J_k$ and Equation~(\ref{align_howard_value_error}),
\begin{align*}
    J_k(s_t^\pi) &= \max_{a\in\mathcal A} \left\{ u_k(s_t^\pi,a) + \gamma J_k(f(s_t^\pi,a)) \right\} \\
    &\leq \max_{a\in\mathcal A} \left\{ u_k(s_t^\pi,a) + \gamma\widetilde J_{I_{\max}}(f(s_t^\pi,a)) \right\}
    + \frac{2RS\gamma^{I_{\max}+1}}{1-\gamma} \\
    &= u_k(s_t^\pi,a_t^\pi) + \gamma\widetilde J_{I_{\max}}(s_{t+1}^\pi) +
    \frac{2RS\gamma^{I_{\max}+1}}{1-\gamma} \\
    &\leq u_k(s_t^\pi,a_t^\pi) + \gamma J_k(s_{t+1}^\pi) + \frac{2RS\gamma^{I_{\max}+1}}{1-\gamma},
\end{align*}
where the last inequality uses $\widetilde J_{I_{\max}}\leq J_k$. Consequently,
\begin{align}
    0 \leq J_k(s_t^\pi) - u_k(s_t^\pi,a_t^\pi) - \gamma J_k(s_{t+1}^\pi)
    \leq \frac{2RS\gamma^{I_{\max}+1}}{1-\gamma}.
    \label{align_howard_bellman_error}
\end{align}
The first inequality follows from Bellman optimality. Thus, approximate planning introduces only the error on
the right-hand side into the Bellman equality used in the proof of Lemma~\ref{lemma:regret_decomposition}.

We now work on the confidence event $\mathcal E$ defined in the proof of
Lemma~\ref{lemma:regret_decomposition}. The same concentration argument applies to the trajectory of the
implemented algorithm, so $\mathbb P(\mathcal E)\geq1-\delta$. Let $E_k$, $L_k$, and
$c_k=\inf_{s\in\mathcal S}J_k(s)$ be as defined in that proof. By Lemma~\ref{lemma:value_span},
\[
    J_k(s)-c_k\in[0,2RSM].
\]
holds for every $s\in\mathcal S$. Centering Equation~(\ref{align_howard_bellman_error}) by $c_k$ and summing
over $t\in E_k$ yields
\begin{align*}
    &L_k(1-\gamma)c_k - \sum_{t\in E_k}u_k(s_t^\pi,a_t^\pi) \\
    \leq& -\left(J_k(s_{\tau_k}^\pi)-c_k\right) - (1-\gamma) \sum_{t=\tau_k+1}^{\tau_{k+1}-1}
    \left(J_k(s_t^\pi)-c_k\right)+ \gamma \left(J_k(s_{\tau_{k+1}}^\pi)-c_k\right) +
    \frac{2RSL_k\gamma^{I_{\max}+1}}{1-\gamma} \\
    \leq& 2RSM + \frac{2RSL_k\gamma^{I_{\max}+1}}{1-\gamma}.
\end{align*}

On $\mathcal E$, optimism still gives $u_k(s,a)\geq r(s,a)$, and hence $J_k(s)\geq J_{\theta^\star}(s)$ and
$c_k\geq c_{\theta^\star}$. Therefore,
\begin{align*}
    L_k(1-\gamma)c_{\theta^\star} - \sum_{t\in E_k}r(s_t^\pi,a_t^\pi) \leq 2RSM +
    \frac{2RSL_k\gamma^{I_{\max}+1}}{1-\gamma} + \sum_{t\in E_k} \left[ u_k(s_t^\pi,a_t^\pi) -
    r(s_t^\pi,a_t^\pi) \right].
\end{align*}
Combining this inequality with Equation~(\ref{align_decomposition_comparator}), summing over episodes, and
recalling that $\sum_{k=1}^{K}L_k=T$ and $T(1-\gamma)=1$, we obtain
\begin{align}
    \operatorname{Reg}_T(\pi) \leq\;& 2(K+2)RSM + \sum_{t=1}^{T} \left[ u_{k(t)}(s_t^\pi,a_t^\pi) -
    r(s_t^\pi,a_t^\pi) \right]+ \frac{2RST\gamma^{I_{\max}+1}}{1-\gamma}.
    \label{align_howard_regret_decomposition}
\end{align}

The determinant and confidence width arguments in the proof of Theorem~\ref{th_upper_bound} remain unchanged.
Using $\overline\Gamma_T$ and $\beta_T$ defined there, Equation~(\ref{align_howard_regret_decomposition})
gives
\begin{align*}
    \operatorname{Reg}_T(\pi) \leq (4\overline\Gamma_T+6)RSM + 4RS\overline\Gamma_T +
    4\beta_T\sqrt{T\overline\Gamma_T}+ \frac{2RST\gamma^{I_{\max}+1}}{1-\gamma}.
\end{align*}
Substituting the definitions of $\overline\Gamma_T$ and $\beta_T$, and using $\gamma=1-1/T$, proves the first
bound.

Finally, if $I_{\max}\geq\lceil2T\log T\rceil$, then
\[
    2RST^2 \left(1-\frac1T\right)^{I_{\max}+1} \leq 2RST^2 \exp\left(-\frac{I_{\max}+1}{T}\right) \leq 2RS.
\]
This additional term is absorbed into $\widetilde O(dRS(M+1))$. The claimed regret order therefore follows
from Theorem~\ref{th_upper_bound}. All bounds hold on $\mathcal E$, completing the proof.
\end{proof}

\section{Technical Lemmas}\label{app_lemma}

In this appendix, we introduce the technical lemmas and their proofs used in our paper.

\begin{lemma}[Likelihood confidence for Algorithm~\ref{al2}]
\label{lem:glbm_likelihood_confidence}
Under the generalized reward model, the event
\[
    \mathcal E:=\bigcap_{t=0}^T
    \left\{\mathcal L_t(\theta^\star)-\min_{\theta\in\Theta}\mathcal L_t(\theta)\le vc_T(\delta)\right\}
\]
satisfies $\mathbb P(\mathcal E)\ge1-\delta$. On $\mathcal E$, the confidence sets in
Equation~(\ref{eq:glbm_confidence_set}) contain $\theta^\star$ at every episode, and
\(0\le u_k(s,a)-r(s,a)\le\Delta_\mu\) for every $k$ and $(s,a)\in\mathcal S\times\mathcal A$.
\end{lemma}

\begin{proof}
Firstly, we want to apply Theorem 3.1 in~\citet{lee2024glm} to the normalized negative log-likelihood
$\mathcal L_t/v$.

To use this theorem, we first note that the unit-norm condition is satisfied by rescaling $\phi_t$ to
$\phi_t/R$ and $\theta$ to $R\theta$. Hence, we only need to verify that $c_T(\delta)$ is a sufficient
threshold. The conditional variance formula and the Cauchy--Schwarz inequality give
\(\mathbb E\left[\left|y_t-\mu(\phi_t^\top\theta^\star)\right|\right]\le\sqrt{vU_\mu}.\) Hence, by Markov's
inequality, with probability at least $1-\delta/2$,
\(\sum_{t=1}^T\left|y_t-\mu(\phi_t^\top\theta^\star)\right|\le\frac{2T\sqrt{vU_\mu}}{\delta}\) holds. On this
event, for every $t\le T$,
\begin{align*}
    \sup_{\theta\in\Theta}\left\|\nabla\left(\mathcal L_t/v\right)(\theta)\right\|_2&\le\frac{R}{v}\left(\sum_{j=1}^t\left|y_j-\mu(\phi_j^\top\theta^\star)\right|+2tU_\mu RS
    \right)
    \\&\le\frac{2RT}{v}\left(U_\mu RS+\frac{\sqrt{vU_\mu}}{\delta}\right)=:G_T.
\end{align*}
Applying Theorem 3.1 in~\citet{lee2024glm} with confidence level $\delta/2$ and intersecting the two events
yields, for all $t\le T$, with probability at least $1-\delta$,
\begin{align*}
    \frac{\mathcal L_t(\theta^\star)-\min_{\theta\in\Theta}\mathcal L_t(\theta)}{v}
    &\le\log\left(\frac{2}{\delta}\right)+d\log\left(\max\left\{\mathrm e,\frac{2\mathrm e
    SG_T}{d}\right\}\right)\\
    &\le\log\left(\frac{2}{\delta}\right)+d+d\log\left(1+\frac{2SG_T}{d}\right)=c_T(\delta).
\end{align*}
This proves $\mathbb P(\mathcal E)\ge1-\delta$. Since $\theta^\star\in\mathcal C_0$, the confidence set update
then gives $\theta^\star\in\mathcal C_k$ by induction, and thus no intersection is empty. The remaining claim
follows from the definition of $u_k$ and the reward range $[\mu(-RS),\mu(RS)]$.
\end{proof}

\begin{lemma}[Prediction widths]
\label{lem:glbm_prediction_widths}
Define \(\beta_1^2=8\kappa vc_T(\delta)+4\lambda_1S^2,\) and
\(\beta_2^2=8\mathrm e\,vc_T(\delta)+4\lambda_2S^2.\) Then on $\mathcal E$, for every round $t$ in episode
$k$, it holds that
\[\omega_k(\phi_t)\le\sqrt{2}\beta_1\|\phi_t\|_{V_{t-1}^{-1}} \text{ and
}\omega_k(\phi_t)\le\sqrt{2}\beta_2\|\phi_t\|_{H_{t-1}^{-1}}.
\]
\end{lemma}

\begin{proof}
Fix an episode $k$ and any $\theta,\theta'\in\mathcal C_k$. By Equation~(\ref{eq:glbm_confidence_set}) and the
optimality of $\widehat\theta_k$,
\begin{align}
    \frac{\mathcal L_{\tau_k-1}(\theta)+\mathcal L_{\tau_k-1}(\theta')}{2}-\mathcal L_{\tau_k-1}\left(\frac{\theta+\theta'}{2}\right)\le vc_T(\delta).\label{eq:glbm_midpoint_bound}
\end{align}
For every $\zeta\in\Theta$, we have
\[
    \nabla^2\mathcal L_{\tau_k-1}(\zeta)=\sum_{j<\tau_k}\mu'(\phi_j^\top\zeta)\phi_j\phi_j^\top\succeq\frac{1}{\kappa}\sum_{j<\tau_k}\phi_j\phi_j^\top.
\]
The midpoint inequality for a convex function with this Hessian lower bound therefore gives
\[\frac{1}{8\kappa}\sum_{j<\tau_k}\left(\phi_j^\top(\theta-\theta')\right)^2\le vc_T(\delta).
\]
Since $\|\theta-\theta'\|_2\le2S$, it holds that
\begin{equation}\label{eq:glbm_V_diameter}
    \|\theta-\theta'\|_{\overline V_k}^2\le8\kappa vc_T(\delta)+4\lambda_1S^2=\beta_1^2.
\end{equation}

For the weighted matrix, integrating $|(\log\mu')'|\le\rho$ gives
\begin{equation}\label{eq:glbm_width_curvature}
    \exp(-\rho|z-z'|)\le\frac{\mu'(z)}{\mu'(z')}\le\exp(\rho|z-z'|),
\end{equation}
for $z,z'\in[-RS,RS]$. The sets $\mathcal C_k$ are convex and nested. Thus, for every $j<\tau_k$, both
$\zeta\in\mathcal C_k$ and $\overline\theta_{k(j)}$ belong to $\mathcal C_{k(j)}$. Whenever $w_j>0$,
\[
    \rho\left|\phi_j^\top(\zeta-\overline\theta_{k(j)})\right|\le\rho\omega_{k(j)}(\phi_j)\le1.
\]
By Equation~(\ref{eq:glbm_width_curvature}), $\mu'(\phi_j^\top\zeta)\ge\mathrm e^{-1}w_j$. The same inequality
holds when $w_j=0$. Therefore, along the segment joining $\theta$ and $\theta'$,
\[\nabla^2\mathcal L_{\tau_k-1}(\zeta)\succeq\mathrm e^{-1}\sum_{j<\tau_k}w_j\phi_j\phi_j^\top.
\]
Applying the midpoint inequality to \Cref{eq:glbm_midpoint_bound} again leads to
\[\frac{1}{8\mathrm e}\sum_{j<\tau_k}w_j\left(\phi_j^\top(\theta-\theta')\right)^2\le vc_T(\delta),\]
and hence
\begin{equation}\label{eq:glbm_H_diameter}
    \|\theta-\theta'\|_{\overline H_k}^2\le8\mathrm e\,vc_T(\delta)+4\lambda_2S^2=\beta_2^2.
\end{equation}
This bound uses the fixed historical weights, rather than weights recomputed from the current estimate.

At any round $t$ in episode $k$, both determinant ratios are less than two. By Lemma 12
in~\citet{AbbasiYadkori2011ImprovedAF}, it holds that
\[\overline V_k^{-1}\preceq2V_{t-1}^{-1},
    \text{ and }\overline H_k^{-1}\preceq2H_{t-1}^{-1}.
\]
Combining these inequalities with Equation~(\ref{eq:glbm_V_diameter}) and (\ref{eq:glbm_H_diameter}) and
applying the Cauchy--Schwarz inequality proves the given width bounds.
\end{proof}

\begin{lemma}[Episodes and regret]
\label{lem:glbm_episodes_regret}
Let $\gamma=1-1/T$. Further define that \(\Gamma_1=d\log\left(1+\frac{TR^2}{\lambda_1d}\right)\) and \(
\Gamma_2=d\log\left(1+\frac{TU_\mu R^2}{\lambda_2d}\right).\) Then the number of episodes $K$ satisfies
\(K\le1+(\Gamma_1+\Gamma_2)/\log2.\) Moreover, conditioned on $\mathcal E$, it holds that
\[\operatorname{Reg}_T\le(K+2)M\Delta_\mu+\sum_{t=1}^T\left[u_{k(t)}(s_t,a_t)-r(s_t,a_t)\right].\]
\end{lemma}

\begin{proof}
Since $\|\phi_t\|_2\le R$ and $0\le w_t\le U_\mu$, the determinant--trace inequality gives
\[\log\left( \frac{\det V_T}{\det V_0} \right)\le\Gamma_1 \text{ and }\log\left( \frac{\det H_T}{\det H_0}
\right)\le\Gamma_2.\]
Note that both matrices are nondecreasing in the positive semidefinite order. Each episode before the last
ends with at least one determinant doubling, thus
\[(K-1)\log2\le\log\left( \frac{\det V_T}{\det V_0} \right)+\log \left(\frac{\det H_T}{\det H_0}
\right)\le\Gamma_1+\Gamma_2.\]

For the regret bound, we use the argument in the proof of Lemma~\ref{lemma:regret_decomposition}. Once
optimism holds, that argument uses only the reward range, the finite-memory transition and exact discounted
planning within each episode. Furthermore, both $r$ and $u_k$ lie in $[\mu(-RS),\mu(RS)]$. Applying the
positive affine transformation
\[h\rightarrow\frac{2RS}{\Delta_\mu}\left(h-\frac{\mu(RS)+\mu(-RS)}{2}\right)
\]
to the rewards maps this interval to $[-RS,RS]$ and preserves the maximizing policies. Regret and
reward-estimation errors are both multiplied by $2RS/\Delta_\mu$. Dividing the bound from the proof of
Lemma~\ref{lemma:regret_decomposition} by this factor gives the decomposition in
\Cref{lem:glbm_episodes_regret}.
\end{proof}

\begin{lemma}\label{lem:glbm_memory_lower}
For $d\geq3$ and $T\geq(d-1)\max\{M,1\}$, there exists an admissible action set such that, for every policy
$\pi$,
\[\sup_{\theta^\star}\mathbb E\left[\mathrm{Reg}_T(\pi)\right]\geq\Omega\left(\frac{d(M+1)\Delta_\mu}{1+\log(1+\rho RS)}\right).\]
\end{lemma}

\begin{proof}
Let $L=RS$ and $p=d-1$. We construct $p$ parameter instances that share the same action set and the same known
memory matrix. In instance $i$, the learner receives a larger mean reward only after reaching a memory state
associated with action $i$. Until that state is reached, the observations do not distinguish instance $i$ from
a common null environment. We first choose the two reward levels used in this construction and relate their
difference to $\Delta_\mu$. Define
\[\Gamma_\mu=\sup_{0<h<2L}\left\{\mu(-L+h)-\mu\left(-L+\frac{3h}{4}\right)\right\}.\]
We will show that $\Gamma_\mu$ is at least a logarithmic fraction of $\Delta_\mu$. Since $\mu'(z)>0$ and
$|\mu''(z)|\leq\rho\mu'(z)$, for any $z,z'\in[-L,L]$,
\[
    \begin{split}
    \left|\log\mu'(z)-\log\mu'(z')\right|=\left|\int_{z'}^z\frac{\mu''(u)}{\mu'(u)}\,du\right|\leq \rho|z-z'|.
    \end{split}
\]
Exponentiating both sides gives
\begin{equation}\label{eq:glbm_lower_slope_comparison}
    e^{-\rho|z-z'|}\mu'(z')\leq \mu'(z) \leq e^{\rho|z-z'|}\mu'(z').
\end{equation}

Set $h_j=2L(3/4)^j$ for $j\geq0$, and let $N$ be the smallest nonnegative integer such that $\rho h_N\leq1$
($N=0$ when $\rho=0$). Since $h_{j+1}=3h_j/4$, the intervals $[-L+3h_j/4,-L+h_j]$, $j=0,\ldots,N$, are
adjacent and together cover $[-L+3h_N/4,L]$. Thus,
\[\begin{split}
    \Delta_\mu=\sum_{j=0}^{N}\left[\mu(-L+h_j)-\mu\left(-L+\frac{3h_j}{4}\right)\right]+\mu\left(-L+\frac{3h_N}{4}\right)-\mu(-L).
    \end{split}\]
It remains to bound the change in $\mu$ over the uncovered interval $[-L,-L+3h_N/4]$. For any
$x\in[-L,-L+3h_N/4]$ and $y\in[-L+3h_N/4,-L+h_N]$, we have $|x-y|\leq h_N$. Consequently,
\Cref{eq:glbm_lower_slope_comparison} and $\rho h_N\leq1$ imply $\mu'(x)\leq e\mu'(y)$. Integrating first over
$y$ gives
\[\frac{h_N}{4}\mu'(x)\leq e\int_{-L+3h_N/4}^{-L+h_N}\mu'(y)\,dy.\]
Integrating this inequality over $x\in[-L,-L+3h_N/4]$ and dividing by $h_N/4$, we obtain
\[\int_{-L}^{-L+3h_N/4}\mu'(x)\,dx\leq
    3e\int_{-L+3h_N/4}^{-L+h_N}\mu'(y)\,dy.\] Therefore,
\[\mu\left(-L+\frac{3h_N}{4}\right)-\mu(-L)\leq
    3e\left[\mu(-L+h_N)-\mu\left(-L+\frac{3h_N}{4}\right)\right].\]

Each term in the preceding sum is at most $\Gamma_\mu$. This also holds for the term with $h_0=2L$: although
$2L$ is excluded from the supremum defining $\Gamma_\mu$, continuity gives
\[\mu(L)-\mu(L/2)=\lim_{h\rightarrow2L}\left[\mu(-L+h)-\mu\left(-L+\frac{3h}{4}\right)\right]\leq \Gamma_\mu.\]
It follows that $\Delta_\mu\leq(N+1+3e)\Gamma_\mu$. Moreover, the definition of $N$ gives
\(N=\left\lceil\frac{\log\bigl(\max\{1,2\rho L\}\bigr)}{\log(4/3)}\right\rceil.\) Hence,
\begin{equation}\label{eq:glbm_lower_range_comparison}
    \Gamma_\mu\geq\frac{\Delta_\mu}{N+1+3e},\quad N=O\left(1+\log(1+\rho L)\right).
\end{equation}

Since $\mu$ is strictly increasing, $\Gamma_\mu>0$. By the definition of the supremum, we can choose
$h\in(0,2L)$ such that
\[\mu(-L+h)-\mu\left(-L+\frac{3h}{4}\right)\geq \frac{\Gamma_\mu}{2}.\]
We further set that \(z_0=-L+\frac{3h}{4}\), \(z_1=-L+h\), and \(g=\mu(z_1)-\mu(z_0)\). Then $-L<z_0<z_1<L$
and $g\geq\Gamma_\mu/2$. We next construct parameters and matrices whose natural parameters are exactly $z_0$
and $z_1$.

Define \(\alpha=\frac{|z_0|+h/4}{L}\), \(\eta=\sqrt{\frac{h}{4\alpha L}}\), \(c=\sqrt{1-\eta^2},\) and let
$\xi=1$ if $z_0\geq0$ and $\xi=-1$ otherwise. To verify that these quantities are well defined, observe that
\[
    \alpha=
    \begin{cases}
    1-h/(2L),&0<h\leq4L/3,\\
    h/L-1,&4L/3<h<2L.
    \end{cases}
\]
Thus $0<\alpha<1$. Also, \(\eta^2=\frac{h/4}{|z_0|+h/4}\in(0,1]\) and \(c^2=\frac{|z_0|}{|z_0|+h/4}\). In
particular, $c^2+\eta^2=1$, and
\begin{equation}\label{eq:glbm_lower_parameter_identities}
    \xi\alpha Lc^2=z_0, \qquad \alpha L\eta^2=\frac h4=z_1-z_0.
\end{equation}
These identities allow us to use the $e_0$ coordinate for a common reward level and coordinate $i$ for the
additional reward in instance $i$. Using the standard basis $e_0,\ldots,e_p$ of $\mathbb R^d$, define
\[
    \theta^{(i)}=S(ce_0+\eta e_i),
\]
for $i\in[p]$, and let $\theta^{(0)}=Sc e_0$. All parameters are admissible, since $\left\| \theta^{(i)}
\right\|_2\leq S$ holds. The parameter $\theta^{(0)}$ will serve as the null environment.

We first consider $M\geq1$. Take
\[
    \mathcal A=\{a_0,a_1,\ldots,a_p\}, \qquad a_0=\alpha e_0, \qquad a_i=\alpha e_0+\sqrt{1-\alpha^2}\,e_i,
    \quad i\in[p],
\]
and initialize the memory at $s_1=a_0^M$. Here, $a_i^M$ denotes the memory state consisting of $M$ copies of
$a_i$. The action $a_0$ is distinct from every $a_i$, $i\in[p]$, because $\alpha<1$. Define the memory matrix
by
\begin{equation}\label{eq:glbm_lower_memory_matrix}
    A(s)^\top=
    \begin{cases}
    R(\xi c e_0+\eta e_j)e_0^\top, &s=a_j^M,\ j\in[p],\\
    R\xi c e_0e_0^\top, &\text{otherwise}.
    \end{cases}
\end{equation}
Furthermore, it can be verified that $\|a_i\|_2\leq 1$ holds for $i\in[p]\cup\{0\}$. The operator norm of a
rank-one matrix $uv^\top$ is $\|u\|_2\|v\|_2$. Hence the two cases in
Equation~(\ref{eq:glbm_lower_memory_matrix}) have operator norms $R$ and $Rc\leq R$, respectively. Moreover,
every action in $\mathcal A$ satisfies $e_0^\top a=\alpha$. If $s=a_j^M$, then
\[
    \begin{split}
    \phi(s,a)^\top\theta^{(i)} &= \alpha R(\xi c e_0+\eta e_j)^\top S(ce_0+\eta e_i)\\
    &= \alpha L\left(\xi c^2+\eta^2\mathbf 1\{i=j\}\right)\\
    &= z_0+(z_1-z_0)\mathbf 1\{i=j\}.
    \end{split}
\]
For every other state, the same inner product equals $\xi\alpha Lc^2=z_0$. Therefore, for all
$a\in\mathcal A$,
\[
    \phi(s,a)^\top\theta^{(i)} =
    \begin{cases}
    z_1,&s=a_i^M,\\
    z_0,&\text{otherwise}.
    \end{cases}
\]
Under $\theta^{(0)}$, the $e_i$ component is absent, so the natural parameter equals $z_0$ at every
state--action pair. Thus, the only way to distinguish instance $i$ from the null environment is to reach
$a_i^M$.

For the first $M$ rounds, the memory still contains at least one copy of the initial action $a_0$. None of the
states $a_i^M$ can therefore occur during these rounds. Every action sequence receives mean reward $\mu(z_0)$
for its first $M$ rounds and at most $\mu(z_1)$ thereafter. Repeating $a_i$ attains this upper bound in
instance $i$, since its memory is $a_i^M$ from round $M+1$ onward. Therefore,
\[
    \mathrm{OPT}_T(\theta^{(i)}) = M\mu(z_0)+(T-M)\mu(z_1) = T\mu(z_0)+(T-M)g.
\]

We now compare the learner's trajectories across environments. Use the same internal random seed for the
learner in every environment. Let $U_1,\ldots,U_T$ be independent uniform random variables, independent of
that seed, and generate each reward observation as
\[
    y_t = F_{\phi(s_t,a_t)^\top\theta}^{-1}(U_t),
\]
where $F_z$ is the cumulative distribution function of the reward with natural parameter $z$, and $F_z^{-1}$
is its inverse. This construction gives the correct reward distribution in each environment. It also ensures
that whenever two environments have the same natural parameter at a given round, they receive the same
observation at that round. On the trajectory under $\theta^{(0)}$, define
\[
    \tau_i = \inf\{t\in[T]:s_t=a_i^M\}, \quad \overline\tau_i = \min\{\tau_i,T+1\},
\]
with $\inf\varnothing=\infty$. Until round $\tau_i$, the trajectories under $\theta^{(i)}$ and $\theta^{(0)}$
coincide. Indeed, they start from the same memory and use the same learner seed. If their histories agree
before some round $t<\tau_i$, they choose the same action, and both natural parameters equal $z_0$. Their
observations are then identical under the coupling, so their histories also agree at the next round. This
proves the claim by induction. Their states and selected actions still agree at round $\tau_i$, but the
observations at that round may differ; no agreement is needed thereafter.

Write $\mathrm{Reg}_T^{(i)}(\pi)$ for the regret on the trajectory under $\theta^{(i)}$. On this trajectory,
the first $\overline\tau_i-1$ mean rewards equal $\mu(z_0)$, and every remaining mean reward is at most
$\mu(z_1)$. Thus,
\[
    \begin{split}
    \sum_{t=1}^{T}r_{\theta^{(i)}}(s_t^\pi,a_t^\pi) &\leq (\overline\tau_i-1)\mu(z_0) +
    (T-\overline\tau_i+1)\mu(z_1)\\
    &= T\mu(z_0)+(T-\overline\tau_i+1)g.
    \end{split}
\]
Subtracting the preceding upper bound from $\mathrm{OPT}_T(\theta^{(i)})$ yields
\[
    \mathrm{Reg}_T^{(i)}(\pi) \geq g(\overline\tau_i-M-1).
\]
The inequality holds for each realization of the common randomness, and hence also after taking expectations.

It remains to count how quickly the null trajectory can visit the $p$ different states. First,
$\tau_i\geq M+1$ whenever $\tau_i$ is finite. Second, if $\tau_i<\tau_j<\infty$ for $i\neq j$, then
$\tau_j-\tau_i\geq M$. Otherwise the two length-$M$ action windows defining $s_{\tau_i}=a_i^M$ and
$s_{\tau_j}=a_j^M$ would overlap, requiring an action in the overlap to equal both $a_i$ and $a_j$, which is
impossible. For each realization, order the truncated hitting times as
$\overline\tau_{(1)}\leq\cdots\leq\overline\tau_{(p)}$. If $\overline\tau_{(j)}\leq T$, the preceding
separation argument gives
\[
    \overline\tau_{(j)} \geq M+1+(j-1)M = jM+1.
\]
If it equals $T+1$, it corresponds to an unvisited state. In either case,
\[
    \overline\tau_{(j)}-(M+1) \geq \min\{(j-1)M,T-M\}.
\]
Therefore,
\begin{align}\label{eq:glbm_lower_memory_average}
    \frac1p\sum_{i=1}^{p} \mathbb E\left[\mathrm{Reg}_T^{(i)}(\pi)\right] &\geq \frac gp\, \mathbb E\left[
    \sum_{j=1}^{p} \bigl(\overline\tau_{(j)}-M-1\bigr) \right] \nonumber\\
    &\geq \frac gp\sum_{j=0}^{p-1}\min\{jM,T-M\} \nonumber\\
    &= \frac gp\sum_{j=0}^{p-1}jM = \frac{(p-1)Mg}{2}.
\end{align}
The last equality uses $T\geq pM$, which implies $T-M\geq(p-1)M$.

For $M=0$, keep the same parameter instances but instead take \(A=RI_d\), \(\mathcal A=\{a_1,\ldots,a_p\}\),
and \(a_i=\alpha(\xi c e_0+\eta e_i)\). These actions are admissible because $\|a_i\|_2=\alpha<1$. For every
$i,j\in[p]$, direct calculation gives
\[
    Ra_j^\top\theta^{(i)} = \alpha L\left(\xi c^2+\eta^2\mathbf 1\{i=j\}\right) =
    z_0+(z_1-z_0)\mathbf 1\{i=j\}.
\]
Thus, action $i$ has mean reward $\mu(z_1)$ in instance $i$, while every other action has mean reward
$\mu(z_0)$. Under $\theta^{(0)}$, all actions again have mean $\mu(z_0)$. The optimal comparator always
selects action $i$, so $\mathrm{OPT}_T(\theta^{(i)})=T\mu(z_1)$.

Use the same coupling as above, and now let $\tau_i$ be the first round in which action $i$ is selected on the
null trajectory. Until that round, the observations in instance $i$ and the null environment coincide. The
first $\overline\tau_i-1$ mean rewards in instance $i$ therefore equal $\mu(z_0)$, which gives
\[
    \begin{split}
    \mathrm{Reg}_T^{(i)}(\pi) &\geq T\mu(z_1) - \left[ (\overline\tau_i-1)\mu(z_0) +
    (T-\overline\tau_i+1)\mu(z_1) \right]\\
    &= g(\overline\tau_i-1).
    \end{split}
\]
Only one action can be selected per round, so the $j$-th distinct action cannot be selected before round $j$.
After ordering the truncated first-visit times, this gives $\overline\tau_{(j)}\geq j$ for $j\in[p]$;
unvisited actions also satisfy this inequality because $T+1>p$. Consequently, for $T\geq p$,
\begin{equation}\label{eq:glbm_lower_initial_average}
    \frac1p\sum_{i=1}^{p} \mathbb E\left[\mathrm{Reg}_T^{(i)}(\pi)\right] \geq \frac gp\sum_{j=1}^{p}(j-1) =
    \frac{(p-1)g}{2}.
\end{equation}

In both cases, the supremum over admissible parameters is at least the average over the $p$ constructed
instances. Thus, Equation~ (\ref{eq:glbm_lower_memory_average}) and (\ref{eq:glbm_lower_initial_average}) give
\[
    \begin{split}
    \sup_{\theta^\star} \mathbb E\left[\mathrm{Reg}_T(\pi)\right] &\geq \frac{(p-1)\max\{M,1\}g}{2}\\
    &\geq \frac{(d-2)(M+1)\Gamma_\mu}{8}\\
    &\geq \frac{(d-2)(M+1)\Delta_\mu}{8(N+1+3e)}.
    \end{split}
\]
Here, the second inequality uses $p=d-1$, $\max\{M,1\}\geq(M+1)/2$, and $g\geq\Gamma_\mu/2$; the last
inequality follows from Equation~(\ref{eq:glbm_lower_range_comparison}). Finally, $d-2\geq d/3$ for $d\geq3$
and $N=O(1+\log(1+\rho RS))$, proving the claimed lower bound.
\end{proof}

\begin{lemma}\label{lem:glbm_statistical_lower}
Let $d\geq2$. If \(T\geq\frac{v(d-1)^2}{U_\mu}\max\left\{\frac{1}{4R^2S^2},\rho^2\right\}\), there exists an
admissible action set such that, with $A(s)=RI_d$ for every $s$, every policy $\pi$ satisfies
\[
    \sup_{\theta^\star} \mathbb E\left[\mathrm{Reg}_T(\pi)\right]
    \geq \frac{3(d-1)}{16\sqrt e}\sqrt{vU_\mu T}.
\]
\end{lemma}

\begin{proof}
Let $L=RS$ and $p=d-1$. We construct a family of instances whose natural parameters lie near a point where
$\mu'$ attains $U_\mu$. In this region, differences in the natural parameters translate into mean-reward
differences proportional to $U_\mu$. At the same time, we choose the instances sufficiently close to one
another so that the learner cannot reliably distinguish their optimal actions within $T$ rounds. We define
\(z_\star\in\arg\max_{|z|\leq L}\mu'(z)\), \(\epsilon=\min\left\{\frac L2,\frac{1}{4\rho}\right\},\) and
\(\eta=\sqrt{\frac{\epsilon}{L}},\) where $1/(4\rho)=\infty$ when $\rho=0$. Such a $z_\star$ exists because
$\mu'$ is continuous on $[-L,L]$. Also, $0<\eta^2\leq1/2$. Set $A(s)=RI_d$ for every $s$, and take
\[
    \mathcal A= \left\{ a_u=\left(\sqrt{1-\eta^2},\frac{\eta}{\sqrt p}u\right): u\in\{-1,1\}^p \right\}.
\]
The first coordinate is shared by all actions, while the remaining $p$ coordinates can independently take
either sign. Every action has unit norm, since
\[
    \|a_u\|_2^2 =1-\eta^2+\frac{\eta^2}{p}\sum_{j=1}^{p}u_j^2 =1.
\]
For each $\omega\in\{-1,1\}^p$, define
\[
    \theta^{(\omega)} =\frac1R\left(z_\star\sqrt{1-\eta^2},\nu\omega\right),
\]
where $\nu=\frac14\sqrt{\frac{vp}{U_\mu\eta^2T}}$. All instances use the same action set and the same known
memory matrix; only the unknown sign vector $\omega$ changes. We first verify that these parameters satisfy
the norm bound. By the definition of $\epsilon$,
\[
    \frac{1}{16\epsilon^2} =\max\left\{\frac{1}{4L^2},\rho^2\right\}.
\]
Thus, the condition on $T$ is equivalent to $T\geq vp^2/(16U_\mu\epsilon^2)$, which gives
\[
    p\nu^2 =\frac{vp^2}{16U_\mu\eta^2T} \leq\frac{\epsilon^2}{\eta^2} =L^2\eta^2.
\]
Consequently,
\[
    \begin{aligned}
    R^2\|\theta^{(\omega)}\|_2^2 &=z_\star^2(1-\eta^2)+p\nu^2\leq L^2(1-\eta^2)+L^2\eta^2 =L^2.
    \end{aligned}
\]
Therefore, $\|\theta^{(\omega)}\|_2\leq S$ for every $\omega$. For action $a_u$ under parameter
$\theta^{(\omega)}$, the natural parameter is
\[
    Ra_u^\top\theta^{(\omega)} =(1-\eta^2)z_\star +\frac{\eta\nu}{\sqrt p}u^\top\omega.
\]
The preceding norm bounds imply
\[
    |Ra_u^\top\theta^{(\omega)}| \leq R\|a_u\|_2\|\theta^{(\omega)}\|_2 \leq L.
\]
Moreover, $|u^\top\omega|\leq p$ and $\nu\sqrt p\leq L\eta$, so
\[
    \begin{aligned}
    \left|Ra_u^\top\theta^{(\omega)}-z_\star\right| &\leq\eta^2|z_\star|+\eta\nu\sqrt p\leq L\eta^2+L\eta^2
    =2\epsilon.
    \end{aligned}
\]
Hence, all the constructed natural parameters lie in $[-L,L]$ and are within $2\epsilon$ of $z_\star$. Since
$\mu'(z_\star)=U_\mu$, the slope comparison in Equation~(\ref{eq:glbm_lower_slope_comparison}) gives
\[
    \mu'(z) \geq e^{-\rho|z-z_\star|}\mu'(z_\star) \geq e^{-2\rho\epsilon}U_\mu \geq e^{-1/2}U_\mu
\]
whenever $z\in[-L,L]$ and $|z-z_\star|\leq2\epsilon$. Together with the definition of $U_\mu$, this yields
\begin{equation}\label{eq:glbm_lower_local_slope}
    e^{-1/2}U_\mu\leq\mu'(z)\leq U_\mu
\end{equation}
for every constructed natural parameter and every point between any two of them. These points all lie in
$[-L,L]$, so the argument also applies when $z_\star$ is an endpoint.

We next bound how well the learner can distinguish two instances that differ in a single coordinate. Let $P_z$
denote the reward distribution with natural parameter $z$. Taking the expectation of the log-likelihood ratio
under $P_z$, and using $\mathbb E_{P_z}[y]=\mu(z)$, gives
\[
    \mathrm{KL}(P_z,P_{z'}) =\frac{\mu(z)(z-z')-b(z)+b(z')}{v}.
\]
Since $b'=\mu$ and $b''=\mu'$, Taylor's formula gives
\[
    \begin{aligned}
    b(z')-b(z)-\mu(z)(z'-z) =(z'-z)^2\int_0^1(1-\alpha)
    \mu'\bigl(z+\alpha(z'-z)\bigr)\,d\alpha\leq\frac{U_\mu}{2}(z'-z)^2,
    \end{aligned}
\]
where we use the fact that the segment between $z$ and $z'$ lies in $[-L,L]$. Therefore,
\begin{equation}\label{eq:glbm_lower_single_kl}
    \mathrm{KL}(P_z,P_{z'}) =\frac{b(z')-b(z)-\mu(z)(z'-z)}{v} \leq\frac{U_\mu}{2v}(z-z')^2.
\end{equation}

Fix any policy $\pi$. Write $\mathbb P_\omega$ for the distribution of its complete action--observation
history under $\theta^{(\omega)}$, and let $\omega^{(j)}$ be obtained by reversing coordinate $j$ of $\omega$.
For every available action $a_u$,
\[
    Ra_u^\top\theta^{(\omega)} -Ra_u^\top\theta^{(\omega^{(j)})} =\frac{2\eta\nu}{\sqrt p}u_j\omega_j.
\]
Thus, the absolute difference between the two natural parameters is always $2\eta\nu/\sqrt p$, regardless of
the selected action. By the divergence decomposition \citep[Lemma~15.1]{lattimore2020bandit} and
Equation~(\ref{eq:glbm_lower_single_kl}),
\[
    \begin{aligned}
    \mathrm{KL}(\mathbb P_\omega,\mathbb P_{\omega^{(j)}})& = \sum_{t=1}^{T}\mathbb E_\omega\left[
    \mathrm{KL}\left( P_{Ra_{u_t}^\top\theta^{(\omega)}}, P_{Ra_{u_t}^\top\theta^{(\omega^{(j)})}} \right)
    \right] \leq T\frac{U_\mu}{2v} \left(\frac{2\eta\nu}{\sqrt p}\right)^2= \frac{2TU_\mu\eta^2\nu^2}{pv}
    =\frac18,
    \end{aligned}
\]
where $a_{u_t}$ denotes the selected action at round $t$. This decomposition compares the reward distributions
at the same action conditional on a given history, and therefore remains valid for an adaptive policy.
Pinsker's inequality now gives
\[
    \|\mathbb P_\omega-\mathbb P_{\omega^{(j)}}\|_{\mathrm{TV}} \leq\sqrt{\frac12\cdot\frac18} =\frac14.
\]

We use this bound to show that the learner selects the wrong sign with a constant probability. Fix a round $t$
and a coordinate $j$, and consider a pair $\omega,\omega^{(j)}$ with $\omega_j=1$. An error in coordinate $j$
occurs when $u_{t,j}=-1$ under $\omega$, or when $u_{t,j}=1$ under $\omega^{(j)}$. Consequently,
\[
    \begin{aligned}
    \mathbb P_\omega(u_{t,j}\neq\omega_j) +\mathbb P_{\omega^{(j)}} \left(u_{t,j}\neq\omega^{(j)}_j\right)& =
    \mathbb P_\omega(u_{t,j}=-1) +\mathbb P_{\omega^{(j)}}(u_{t,j}=1)\\
    &= 1+\mathbb P_\omega(u_{t,j}=-1) -\mathbb P_{\omega^{(j)}}(u_{t,j}=-1)\\
    &\geq 1-\|\mathbb P_\omega-\mathbb P_{\omega^{(j)}}\|_{\mathrm{TV}} \geq\frac34.
    \end{aligned}
\]
The inequality follows because the difference in the probability of the same event under two distributions is
bounded by their total variation distance. For each coordinate $j$, the $2^p$ sign vectors form $2^{p-1}$
disjoint pairs of this type. Summing over these pairs and dividing by $2^p$ gives that, for every round $t$
and every coordinate $j$,
\begin{equation}\label{eq:glbm_lower_coordinate_error}
    \frac1{2^p}\sum_{\omega\in\{-1,1\}^p} \mathbb P_\omega(u_{t,j}\neq\omega_j) \geq\frac{2^{p-1}}{2^p}\frac34
    =\frac38.
\end{equation}

Finally, we relate these coordinate errors to regret. Since $\mu$ is increasing, an optimal action maximizes
$u^\top\omega$, which is achieved by $u=\omega$. Thus, $a_\omega$ is optimal under $\theta^{(\omega)}$. For
any selected action $a_{u_t}$,
\[
    \begin{aligned}
    R(a_\omega-a_{u_t})^\top\theta^{(\omega)} =\frac{\eta\nu}{\sqrt p}
    \sum_{j=1}^{p}(1-u_{t,j}\omega_j)=\frac{2\eta\nu}{\sqrt p} \sum_{j=1}^{p}\mathbf 1\{u_{t,j}\neq\omega_j\}.
    \end{aligned}
\]
In particular, each incorrect coordinate contributes $2\eta\nu/\sqrt p$ to the difference in natural
parameters. The slope lower bound in Equation~(\ref{eq:glbm_lower_local_slope}) holds throughout the interval
between these parameters. Integrating $\mu'$ over this interval therefore gives
\[
    \begin{aligned}
    \mu(Ra_\omega^\top\theta^{(\omega)}) -\mu(Ra_{u_t}^\top\theta^{(\omega)}) \geq \frac{U_\mu}{\sqrt e}
    R(a_\omega-a_{u_t})^\top\theta^{(\omega)} = \frac{2U_\mu\eta\nu}{\sqrt{ep}}
    \sum_{j=1}^{p}\mathbf 1\{u_{t,j}\neq\omega_j\}.
    \end{aligned}
\]

Because $A(s)=RI_d$, the memory state does not affect rewards, and the optimal comparator selects $a_\omega$
in every round. Summing the preceding inequality over $T$ rounds, taking expectations, and averaging over the
parameter instances yields
\begin{align*}
    \frac1{2^p}\sum_{\omega\in\{-1,1\}^p} \mathbb E_\omega\left[\mathrm{Reg}_T(\pi)\right]
    &\geq \frac{2U_\mu\eta\nu}{\sqrt{ep}} \sum_{t=1}^{T}\sum_{j=1}^{p} \frac1{2^p}\sum_{\omega\in\{-1,1\}^p}
    \mathbb P_\omega(u_{t,j}\neq\omega_j)\\
    &\geq \frac{2U_\mu\eta\nu}{\sqrt{ep}}\frac{3pT}{8}= \frac{3p}{16\sqrt e}\sqrt{vU_\mu T}.
\end{align*}
The second inequality follows from \Cref{eq:glbm_lower_coordinate_error}, and the last equality uses the
definition of $\nu$. Since all the constructed parameters satisfy $\|\theta^{(\omega)}\|_2\leq S$, the
supremum over admissible parameters is at least this average. Substituting $p=d-1$ proves the lemma.
\end{proof}

\end{APPENDICES}
\end{document}